\documentclass[final]{styles/colt2021/colt2021} % Anonymized submission

\usepackage{times}
\usepackage{booktabs}       % professional-quality tables
\usepackage{amsfonts}       % blackboard math symbols
\usepackage{nicefrac}       % compact symbols for 1/2, etc.
\usepackage{enumitem}
\usepackage{nicefrac}
\PassOptionsToPackage{linktocpage}{hyperref} % test line break

\usepackage{amsmath,amssymb,mathrsfs}

\usepackage{footnote}
\usepackage{algorithm}
\usepackage{algorithmic}

\allowdisplaybreaks

\makeatletter
\def\set@curr@file#1{\def\@curr@file{#1}} %temp workaround for 2019 latex release; otherwise the command \includegraphics clashes...
\makeatother

\usepackage{graphicx}
\usepackage{mathtools}
\usepackage{footnote}
\usepackage{float}
\usepackage{xspace}
\usepackage{multirow}
\usepackage{wrapfig}
\usepackage{framed}
\usepackage{bbm}
\usepackage{footnote}
\usepackage{nicefrac}
\usepackage{makecell}
\usepackage[T1]{fontenc}
\makesavenoteenv{tabular}
\makesavenoteenv{table}

\definecolor{Green}{rgb}{0.13, 0.65, 0.3}
\usepackage[]{color-edits}

\addauthor{HL}{red}
\addauthor{LC}{cyan}

\newcommand{\calA}{{\mathcal{A}}}

\newcommand{\calS}{{\mathcal{S}}}
\newcommand{\calF}{{\mathcal{F}}}

\newcommand{\KL}{\text{\rm KL}}

\DeclareMathOperator*{\argmin}{argmin}

\DeclarePairedDelimiter\abs{\lvert}{\rvert}

\newcommand{\eat}[1]{}

\newcommand{\inner}[2]{\left\langle #1,#2 \right\rangle}
\newcommand{\inners}[2]{\langle #1,#2\rangle}

\newcommand{\rbr}[1]{\left(#1\right)}
\newcommand{\sbr}[1]{\left[#1\right]}

\newcommand{\bigO}[1]{\order\left( #1 \right)}
\newcommand{\tilO}[1]{\otil\left( #1 \right)}
\newcommand{\lowO}[1]{\lorder\left( #1 \right)}
\newcommand{\bigo}[1]{\order( #1 )}
\newcommand{\tilo}[1]{\otil( #1 )}

\DeclarePairedDelimiter\ceil{\lceil}{\rceil}

\newcommand{\Tmax}{\ensuremath{T_{\max}}}
\newcommand{\smax}{\ensuremath{s_{\max}}}
\newcommand{\T}{\ensuremath{T_\star}}
\newcommand{\cmin}{\ensuremath{c_{\min}}}
\newcommand{\propers}{\ensuremath{\Pi_{\rm proper}}}

\newcommand{\dev}{\textsc{Dev}}
\newcommand{\reg}{\textsc{Reg}}

\newcommand{\var}{\textsc{Var}}

\newcommand{\bernoulli}{\textrm{Bernoulli}}
\newcommand{\aux}{\chi}

\newcommand{\PSD}{\propers}
\newcommand{\SA}{\Gamma}
\newcommand{\qfeat}{\phi} % (1+\lambda h)q

\newcommand{\regz}{\psi} % regularizer

\newcommand{\sums}{\sum_{s\in\calS}}
\newcommand{\suma}{\sum_{a\in\calA_s}}
\newcommand{\sumsaf}[1][s, a]{\sum_{(#1)\in\SA}}
\newcommand{\sumtilsaf}[1][s, a]{\sum_{(#1)\in\tilSA}}

\newcommand{\sumsa}[1][s, a]{\sum_{(#1)}}

\newcommand{\sumh}{\sum_{h=1}^H}

\newcommand{\sumk}{\sum_{k=1}^K}

\newcommand{\hatb}{\widehat{b}}
\newcommand{\hatc}{\widehat{c}}

\newcommand{\optpi}{\pi^\star}

\newcommand{\tiloptq}{q_{\tiloptpi}}

\newcommand{\tilN}{\widetilde{N}}

\newcommand{\tils}{\widetilde{s}}

\newcommand{\tilc}{\widetilde{c}}
\newcommand{\tilf}{\widetilde{f}}

\newcommand{\tilM}{\widetilde{M}}
\newcommand{\tilP}{\widetilde{P}}
\newcommand{\tilA}{\widetilde{\calA}}
\newcommand{\tilS}{\widetilde{\calS}}
\newcommand{\tilSA}{\widetilde{\SA}}
\newcommand{\tilpi}{{\widetilde{\pi}}}
\newcommand{\tiloptpi}{{\widetilde{\pi}^\star}}
\newcommand{\tilDelta}{{\widetilde{\Delta}}}

\newcommand{\h}[1]{\vec{h}\circ #1}
\newcommand{\field}[1]{\mathbb{#1}}

\newcommand{\fR}{\field{R}}

\newcommand{\fN}{\field{N}}
\newcommand{\E}{\field{E}}

\newcommand{\Ind}{\field{I}}

\newcommand{\order}{\ensuremath{\mathcal{O}}}
\newcommand{\lorder}{\ensuremath{\Omega}}
\newcommand{\otil}{\ensuremath{\tilde{\mathcal{O}}}}

\usepackage{prettyref}
\newcommand{\pref}[1]{\prettyref{#1}}
\newcommand{\pfref}[1]{Proof of \prettyref{#1}}
\newcommand{\savehyperref}[2]{\texorpdfstring{\hyperref[#1]{#2}}{#2}}
\newrefformat{eq}{\savehyperref{#1}{Eq.~\textup{(\ref*{#1})}}}
\newrefformat{eqn}{\savehyperref{#1}{Equation~\ref*{#1}}}
\newrefformat{lem}{\savehyperref{#1}{Lemma~\ref*{#1}}}
\newrefformat{def}{\savehyperref{#1}{Definition~\ref*{#1}}}
\newrefformat{line}{\savehyperref{#1}{Line~\ref*{#1}}}
\newrefformat{thm}{\savehyperref{#1}{Theorem~\ref*{#1}}}
\newrefformat{corr}{\savehyperref{#1}{Corollary~\ref*{#1}}}
\newrefformat{cor}{\savehyperref{#1}{Corollary~\ref*{#1}}}
\newrefformat{sec}{\savehyperref{#1}{Section~\ref*{#1}}}
\newrefformat{subsec}{\savehyperref{#1}{Section~\ref*{#1}}}
\newrefformat{app}{\savehyperref{#1}{Appendix~\ref*{#1}}}
\newrefformat{assum}{\savehyperref{#1}{Assumption~\ref*{#1}}}
\newrefformat{ex}{\savehyperref{#1}{Example~\ref*{#1}}}
\newrefformat{fig}{\savehyperref{#1}{Figure~\ref*{#1}}}
\newrefformat{alg}{\savehyperref{#1}{Algorithm~\ref*{#1}}}
\newrefformat{rem}{\savehyperref{#1}{Remark~\ref*{#1}}}
\newrefformat{conj}{\savehyperref{#1}{Conjecture~\ref*{#1}}}
\newrefformat{prop}{\savehyperref{#1}{Proposition~\ref*{#1}}}
\newrefformat{proto}{\savehyperref{#1}{Protocol~\ref*{#1}}}
\newrefformat{prob}{\savehyperref{#1}{Problem~\ref*{#1}}}
\newrefformat{claim}{\savehyperref{#1}{Claim~\ref*{#1}}}
\newrefformat{que}{\savehyperref{#1}{Question~\ref*{#1}}}
\newrefformat{op}{\savehyperref{#1}{Open Problem~\ref*{#1}}}
\newrefformat{fn}{\savehyperref{#1}{Footnote~\ref*{#1}}}
\newrefformat{tab}{\savehyperref{#1}{Table~\ref*{#1}}}
\newrefformat{fig}{\savehyperref{#1}{Figure~\ref*{#1}}}

\DeclareOldFontCommand{\rm}{\normalfont\rmfamily}{\mathrm}
\DeclareOldFontCommand{\it}{\normalfont\itshape}{\mathit}
\title[Minimax Regret for Stochastic Shortest Path]{Minimax Regret for Stochastic Shortest Path with \\ Adversarial Costs and Known Transition}
\usepackage{times}
\coltauthor{%
 \Name{Liyu Chen} \Email{liyuc@usc.edu}\\
 \addr University of Southern California
 \AND
 \Name{Haipeng Luo} \Email{haipengl@usc.edu}\\
 \addr University of Southern California
 \AND
 \Name{Chen-Yu Wei} \Email{chenyu.wei@usc.edu}\\
 \addr University of Southern California
}

\begin{document}
\SetAlgoVlined
\DontPrintSemicolon
\maketitle

\begin{abstract}
We study the stochastic shortest path problem with adversarial costs and known transition, and show that the minimax regret is $\tilo{\sqrt{D\T K}}$ and $\tilo{\sqrt{D\T SA K}}$ for the full-information setting and the bandit feedback setting respectively,
where $D$ is the diameter, $\T$ is the expected hitting time of the optimal policy, $S$ is the number of states, $A$ is the number of actions, and $K$ is the number of episodes.
Our results significantly improve upon the recent work of~\citep{rosenberg2020adversarial} which only considers the full-information setting and achieves suboptimal regret.
Our work is also the first to consider bandit feedback with adversarial costs.

Our algorithms are built on top of the Online Mirror Descent framework with a variety of new techniques that might be of independent interest, including an improved multi-scale expert algorithm, a reduction from general stochastic shortest path to a special loop-free case, a skewed occupancy measure space, %the usage of log-barrier with an increasing learning rate schedule, 
and a novel correction term added to the cost estimators.
Interestingly, the last two elements reduce the variance of the learner via positive bias and the variance of the optimal policy via negative bias respectively,
and having them simultaneously is critical for obtaining the optimal high-probability bound in the bandit feedback setting.
\end{abstract}

% https://latex.org/forum/viewtopic.php?t=21170
\tolerance 1414
\hbadness 1414
\emergencystretch 1.5em
\hfuzz 0.3pt
\widowpenalty=10000
\vfuzz \hfuzz
\raggedbottom

\sloppy

\section{Introduction}
% !TEX root = main.tex

We study the stochastic shortest path (SSP) problem, where a learner tries to reach a goal state in a Markov Decision Process (MDP) with minimum total cost.
The problem proceeds in $K$ episodes. 
In each episode, the learner starts from a fixed state, sequentially selects an available action, incurs a cost, and transits to the next state sampled from a fixed transition function.
The episode ends when the learner reaches a fixed goal state.
The performance of the learner is measured by her regret, which is the difference between her total cost over the $K$ episodes and that of the best fixed policy.

The special case of SSP where the learner is guaranteed to reach the goal state within a fixed number of steps is extensively studied in recent years.
It is often known as episodic finite-horizon reinforcement learning or equivalently loop-free SSP.
The general case, however, is much less understood.
Recently, \citet{tarbouriech2019no} and~\citet{cohen2020near} study the case where the costs are fixed or generated stochastically and develop algorithms with sub-linear regret.
Another recent work by~\citet{rosenberg2020adversarial} considers adversarial costs that are chosen arbitrarily but revealed at the end of each episode (the so-called full-information setting).
When the transition function is known, their algorithm achieves $\tilo{\frac{D}{\cmin}\sqrt{K}}$ regret where $D$ is the diameter of the MDP and $\cmin\in(0,1]$ is a global lower bound of the cost for any state-action pair.
When $\cmin=0$, they provide a different algorithm with regret $\tilo{\sqrt{D\T}K^{3/4}}$ where $\T$ is the expected time for the optimal policy to reach the goal state.
They also further study the case with unknown transition.

In this work, we significantly improve the state-of-the-art for the general SSP problem with adversarial costs and known transition, by developing matching upper and lower bounds for both the full-information setting and the bandit feedback setting.
More specifically, our results are (see also \pref{tab:known} for a summary):
\begin{itemize}
\item
In the full-information setting, we show that the minimax regret is of order $\Theta(\sqrt{D\T K})$ (ignoring logarithmic terms), with no dependence on $1/\cmin$ (it can be shown that $\T \leq D/\cmin$).
We develop two algorithms, one with optimal expected regret (\pref{alg:meta}) and another with optimal high probability regret (\pref{alg:full-known-hp}).
Note that, as pointed out by~\citet{rosenberg2020adversarial}, achieving high probability bounds for SSP is significantly more challenging even in the full-information setting, since the learner is often not guaranteed to reach the goal within a fixed number of steps with high probability.
We complement our algorithms and upper bounds with a matching lower bound in \pref{thm:full-known-lb}.

\item
Next, we further consider the more challenging bandit feedback setting where the learner only observes the cost for the visited state-action pairs, which has not been studied before in the adversarial cost case to the best of our knowledge.
We show that the minimax regret is of order $\Theta(\sqrt{D\T SAK})$ (ignoring logarithmic terms) where $S$ is the number of states and $A$ is the number of actions.
We again developed two algorithms, one with optimal expected regret (\pref{alg:bandit-known-pseudo}) and another more complex one with optimal high probability regret (\pref{alg:bandit-known-hp}).
A matching lower bound is shown in \pref{thm:bandit-known-lb}.
\end{itemize}

\renewcommand{\arraystretch}{1.2}
\begin{table}[t]
	\centering
	\caption{Summary of our minimax optimal results and comparisons with prior work. Here, $D, S, A$ are the diameter, number of states, and number of actions of the MDP, $\cmin$ is the minimum cost, $\T\leq D/\cmin$ is the expected hitting time of the optimal policy, and $K$ is the number of episodes. Logarithmic terms are omitted. All algorithms can be implemented efficiently. \pref{alg:meta} is completely parameter-free, while others require the knowledge of $\T$.}
	\label{tab:known}
	\resizebox{\textwidth}{!}{%
	\begin{tabular}{|c |c | c | }
		\hline
		 & \multicolumn{1}{c|}{Minimax Regret (\textbf{this work})} & \multicolumn{1}{c|}{\citep{rosenberg2020adversarial}} \\
		 \hline
		 \multirow{4}{*}{Full information} & $\Theta(\sqrt{D\T K})$   & \multirow{4}{*}{$\tilO{\frac{D}{\cmin}\sqrt{K}}$ or $\tilO{\sqrt{D\T}K^{\frac{3}{4}}}$} \\
		 & \multicolumn{1}{l|}{\pref{alg:meta} (expected bound)} &  \\
		 & \multicolumn{1}{l|}{\pref{alg:full-known-hp} (high probability bound)} &   \\
		 & \multicolumn{1}{l|}{\pref{thm:full-known-lb} (lower bound)} &  \\
		 \hline
		 \multirow{4}{*}{Bandit feedback} & \multicolumn{1}{c|}{$\Theta(\sqrt{D\T SAK})$}  & \multirow{4}{*}{N/A} \\
		 & \multicolumn{1}{l|}{\pref{alg:bandit-known-pseudo} (expected bound)} &  \\
		 & \multicolumn{1}{l|}{\pref{alg:bandit-known-hp} (high probability bound)} &  \\
		 & \multicolumn{1}{l|}{\pref{thm:bandit-known-lb} (lower bound)} &  \\
		 \hline
	\end{tabular}
	}
\end{table}

\paragraph{Techniques}
Similarly to~\citep{rosenberg2020adversarial}, our algorithms are all based on the standard Online Mirror Descent (OMD) framework.
However, a variety of new techniques are developed on top of OMD to achieve our results.
For example, to obtain the optimal expected regret in the full-information setting without knowing $\T$ ahead of time, we reduce the problem to the {\it multi-scale expert}  problem studied in~\citep{bubeck2017online,foster2017parameter,cutkosky2018black} and develop a new algorithm with an improved guarantee necessary to achieve our results,
which might be of independent interest.\footnote{See also concurrent work~\citep{chen2021impossible} by the same authors for in-depth discussions and significant extensions of this idea.}

Our other algorithms all require a reduction from a general SSP instance to its loop-free version (\pref{def:reduction}) as well as executing OMD over a {\it skewed occupancy measure space}, both of which are novel as far as we know.
The skewed occupancy measure can be viewed as adding {\it positive} bias to the costs, as a way to reduce the variance of the learner.
These algorithms require setting some parameters in terms of $\T$ to achieve the optimal regret though (see discussions after \pref{thm:full-known-hp}).

In addition, the two algorithms in the bandit feedback setting require the usage of the log-barrier regularizer, an increasing learning rate schedule similar to~\citep{lee2020bias}, and injecting another {\it negative} bias term into the cost estimator to reduce the variance of the optimal policy.
We find the necessity of both positive and negative bias in the bandit setting intriguing.

\paragraph{Related work}
Earlier research studies SSP as a control problem and focuses on finding the optimal policy efficiently with all the parameters known; see for example~\citep{bertsekas1991analysis,bertsekas2013stochastic}.
Learning with low regret in SSP was first studied in~\citep{tarbouriech2019no},
which considers fixed or stochastic costs and proposes algorithms with sub-linear regret that depends on $1/\cmin$.
\citet{cohen2020near} remove the $1/\cmin$ dependence and propose an algorithm with almost optimal regret.
Note that their bounds do not depend on the parameter $\T$; see our discussions after \pref{thm:full-known-lb} on why $\T$ shows up in the adversarial cost case.

To the best of our knowledge, \citep{rosenberg2020adversarial} is the only existing work that studies SSP with adversarial costs.
They only study the full-information setting, with either known or unknown transition,
while we consider both the full-information setting and the bandit feedback setting, but only with known transition.
We note that our loop-free reduction is readily applied to the unknown transition case, but it only leads to some suboptimal bounds (details omitted).

As mentioned, the special case of SSP with a fixed horizon is extensively studied in recent years, for both stochastic costs (see e.g., \citep{azar2017minimax,jin2018q,zanette2019tighter,efroni2020optimistic}) and adversarial costs (see e.g.,  \citep{neu2012adversarial,zimin2013online,rosenberg2019online,jin2019learning}).
The latter also heavily relies on the OMD framework, but the occupancy measure space that OMD operates over is much simpler compared to general SSP.
Note that, although one of our key algorithmic ideas is to reduce general SSP to this special case, it does not mean that one can directly apply these existing algorithms after the reduction, as it only leads to suboptimal bounds.
Instead, one must further utilize different properties of the original SSP instance to achieve the minimax regret, as we will discuss in detail.

\section{Preliminaries}\label{sec:prelim}
% !TEX root = main.tex

A stochastic shortest path (SSP) instance is defined by an MDP $M=\left( \calS, s_0, g, \calA, P \right)$ and a sequence of $K$ cost functions $\{c_k\}_{k=1}^K$.
Here, $\calS$ is a finite state space,
$s_0 \in \calS$ is the initial state,
$g \notin \calS$ is the goal state,
and $\calA=\{\calA_s\}_{s\in\calS}$ is a finite action space where $\calA_s$ is the available action set at state $s$.
We denote by $\SA = \{(s,a): s\in \calS, a \in \calA_s\}$ the set of valid state-action pairs, and by $S=\lvert\calS\rvert$ and $A=(\sums |\calA_s|)/S$ the total number of states and the average number of available actions respectively.
The transition function $P: \SA\times\calS\cup\{g\}\rightarrow [0, 1]$ is such that $P(s'| s, a)$ is the probability of transiting to $s'$ after taking action $a \in \calA_s$ at state $s$, and it satisfies $\sum_{s'\in\calS\cup\{g\}}P(s'|s, a)=1$ for each $(s,a) \in \SA$. 
Finally, the cost function $c_k: \SA \rightarrow [0,1]$ specifies the cost for each state-action pair during episode $k$. %and is selected by an adversary ahead of time knowing the learner's algorithm. \footnote{%
%We consider this so-called oblivious setting for simplicity,
%but our results hold for the more general adaptive setting as well where each $c_k$ can depend on the randomness before episode $k$
%}
%We do not make any assumptions on how the cost functions are generated, that is, they can be chosen arbitrarily.

The learning protocol is as follows.
The learner interacts with a known MDP $M$ through $K$ episodes. In each episode $k = 1, \ldots, K$, the environment adaptively decides the cost function $c_k$, which can depend on the learner's algorithm and the randomness before episode $k$.
Simultaneously, 
starting from the initial state $s_0\in\calS$, 
the learner sequentially selects an action and transits to the next state according to the transition function, until reaching the goal state $g$.
More formally, in each step $i$ of the episode, the learner observes its current state $s_k^i$ (with $s_k^1 = s_0$ always).
If $s_k^i \neq g$, the learner selects an action $a_k^i \in \calA_{s_k^i}$ and moves to the next state $s_k^{i+1}$ sampled from $P(\cdot| s_k^i, a_k^i)$.
The episode ends when the current state is the goal state,
and we denote by $I_k$ the number of steps in this episode such that $s_k^{I_k+1} = g$.

We consider two different types of feedback on the cost functions for the learner after the goal state is reached.
In the {\it full-information} setting, the entire cost function $c_k$ is revealed to the learner,
while in the {\it bandit feedback} setting, only the costs for the visited state-action pairs, that is, $c_k(s_k^i, a_k^i)$ for $i = 1, \ldots, I_k$, are revealed to the learner.

\paragraph{Proper policies and related concepts}
Before discussing the goal of the learner, we introduce several necessary concepts.
A stationary policy is a mapping $\pi$ such that $\pi(a|s)$ specifies the probability of taking action $a \in \calA_s$ in state $s$.
It is deterministic if $\pi(\cdot|s)$ concentrates on one single action (denoted by $\pi(s)$) for all $s$.
It is {\it proper} if executing it in the MDP starting from any state ensures that the goal state is reached within a finite number of steps with probability $1$ (otherwise it is called improper).
The set of all deterministic and proper policies is denoted by \propers.
Following~\citep{rosenberg2020adversarial}, we make the basic assumption $\propers\neq \emptyset$.

Let $T^{\pi}(s)$ denote the expected hitting time it takes for $\pi$ to reach $g$ starting from state $s$.
If $\pi$ is proper, then $T^{\pi}(s)<\infty$ for any state $s$.
%On the other hand, if $\pi$ is improper, then there exists a state $s$, such that $T^{\pi}(s)=\infty$.
The \textit{fast policy} $\pi^f$ is the (deterministic) policy that achieves the minimum expected hitting time starting from any state,
and the diameter of the MDP is defined as $D=\max_{s\in\calS}\min_{\pi\in\PSD}T^{\pi}(s)=\max_{s\in\calS} T^{\pi^f}(s)$.
%\pref{assum:communicating} implies that $D<\infty$.
Note that both $\pi^f$ and $D$ can be computed ahead of time since we consider the known transition setting.

Given a cost function $c$ and a proper policy $\pi$, we define the \textit{cost-to-go function} $J^{\pi}:\calS\rightarrow [0,\infty)$ such that 
$J^{\pi}(s)=\E\left[\left.\sum_{i=1}^{I}c(s^i, a^i)\right| P, \pi, s^1=s\right]$,
where the expectation is over the randomness of the action $a^i$ drawn from $\pi(\cdot|s^i)$, the state $s^{i+1}$ drawn from $P(\cdot|s^i, a^i)$, and the number of steps $I$ before reaching $g$.
%For a proper policy $\pi$, $J^{\pi}(s)<\infty$ for any state $s$.
%However, $J^{\pi}(s)$ could be finite even if $\pi$ is improper.
%\citet{bertsekas1991analysis} show that every proper policy $\pi$ satisfies:
%\begin{align*}
%	J^{\pi}(s) &= \suma\pi(a| s)\left( c(s, a) + \sumsp P(s'| s, a)J^{\pi}(s') \right),\\
%	T^{\pi}(s) &= 1 + \suma\sumsp\pi(a| s)P(s'| s, a)T^{\pi}(s').
%\end{align*}
%We also define the $Q$ value function for any proper policy $\pi$ as:
%\begin{align*}
%	Q^{\pi}(s, a) &= \E\left[\left.\sum_{i=1}^I c(s^i, a^i)\right| P, \pi, s^1=s, a^1=a \right] = c(s, a) + \sumsp P(s'| s, a)J^{\pi}(s').
%\end{align*}
%Similarly, define $T^{\pi}(s, a)=1 + \sum_{s'} P(s'|s, a)T^{\pi}(s')$.
We use $J_k^\pi$ to denote the cost-to-go function with respect to the cost $c_k$.

\paragraph{Learning objective}
The learner's goal is to minimize her {\it regret}, defined as the difference between her total cost and the total expected cost of the best deterministic proper policy in hindsight:
$
	R_K = \sumk\sum_{i=1}^{I_k}c_k(s_k^i, a_k^i) - \sumk J^{\optpi}_k(s_0),
$
where $\optpi\in\argmin_{\pi\in\PSD}\sumk J^{\pi}_k(s_0)$. %If there is an episode $k$ such that $I_k=\infty$, we define $R_K=\infty$.
By the Markov property, it is clear that $\optpi$ is in fact also the optimal policy starting from any other state, that is, $\optpi\in\argmin_{\pi\in\PSD}\sumk J^{\pi}_k(s)$  for any $s \in \calS$.
Two quantities related to $\optpi$ play an important role in our analysis:
its expected hitting time starting from the initial state $\T = T^{\optpi}(s_0)$
and its largest expected hitting time starting from any state $\Tmax = \max_s T^{\optpi}(s)$.
Let $\cmin = \min_{k}\min_{(s,a)} c_k(s,a)$ be the minimum cost, and $\smax \in \calS$ be such that $\Tmax = T^{\optpi}(\smax)$. We have
$\Tmax \cmin \leq J^{\optpi}_k(\smax)$ and $J^{\pi^f}_k(\smax) \leq D$ by definition.
Together with the fact $\sumk J^{\optpi}_k(\smax) \leq \sumk J^{\pi^f}_k(\smax)$,
this implies $\T \leq \Tmax \leq \frac{D}{\cmin}$ if $\cmin > 0$ (which is one of the reasons why $\cmin$ shows up in existing results). 

\paragraph{Occupancy measure}
For a fixed MDP, a proper policy $\pi$ induces an occupancy measure $q_{\pi} \in \fR_{\geq 0}^{\SA}$ such that $q_{\pi}(s, a)$ is the expected number of visits to $(s, a)$ when executing $\pi$, that is:
$
	q_{\pi}(s, a) = \E\left[\left. \sum_{i=1}^I \Ind\{s^i=s, a^i=a\}\right| P, \pi, s^1=s_0 \right].
$
Similarly, $q_{\pi}(s) = \suma q_{\pi}(s, a)$ is the expected number of visits to $s$ when executing $\pi$.
%When there is no confusion, we ignore the dependency on $P$ and write $q_{P, \pi}$ as $q_{\pi}$.
Clearly, we have $J^{\pi}_k(s_0) = \sum_{(s,a)\in\SA} q_{\pi}(s, a)c_k(s, a) = \inner{q_{\pi}}{c_k}$,
and if the learner executes a stationary proper policy $\pi_k$ in episode $k$, then the expected regret can be written as 
\begin{equation}\label{eq:regret_linear_form}
\E[R_K] = \E\left[\sumk J^{\pi_k}_k(s_0) - J^{\optpi}_k(s_0) \right] = \E\left[ \sumk\inner{q_{\pi_k} - q_{\optpi}}{c_k} \right],
\end{equation}
converting the problem into a form of online linear optimization and making Online Mirror Descent a natural solution to the problem.
%Our algorithms all make use of this framework, but differ in the set of occupancy measures that the algorithm operates on.
Note that, given a function $q: \SA \rightarrow [0, \infty)$, if it corresponds to an occupancy measure, then the corresponding policy $\pi_q$ can clearly be obtained by $\pi_q(a| s) \propto q(s,a)$.
Also note that $T^{\pi}(s_0) = \sumsa q_\pi(s,a) = \sums q_\pi(s)$.

%Note that if $\pi$ is proper, then $q_{\pi}(s, a)<\infty,\forall s, a$.
%Moreover, the mapping between proper policies and finite occupancy measures is bijective, and its inverse for an occupancy measure $q$ is given by $\pi_q(a| s)=q(s, a)/q(s)$, where $q(s)=\suma q(s, a)$.

%Using the equivalence between proper policies and occupancy measures, the SSP problem can be transformed into an online linear optimization problem.
%Notice that the cost-to-go function of $\pi$ is linear w.r.t $q_{\pi}$:
%\begin{align*}
%	J^{\pi}_k(s_0) = \sums\suma q_{\pi}(s, a)c_k(s, a) = \inner{q_{\pi}}{c_k}.
%\end{align*}
%Therefore, the expected regret can also be written as follows:
%\begin{align*}
%	\E[R_K] = \E\left[\sumk J^{\pi_k}_k(s_0) - J^{\optpi}_k(s_0) \right] = \E\left[ \sumk\inner{q_{\pi_k} - q_{\optpi}}{c_k} \right].
%\end{align*}

\paragraph{Other notations}
%Define $\Ind_k(s, a) = \Ind\{\exists i=1,\ldots, I_k, \;(s_k^i, a_k^i)=(s, a)\}$, $\Ind_k(s, a, i) = \Ind\{s_k^i=s, a_k^i=a\}$, and $N_k(s, a)=\sum_{i=1}^{I_k}\Ind_k(s, a, i)$ (the number of visits to pair $(s,a)$ in episode $k$).
%Similarly, for a policy $\pi$, 
%define $N_{\pi}(s, a)=\sum_{i=1}^{I}1\{s^i=s, a^i=a\}$ as the (random) number of visits to pair $(s,a)$ when executing policy $\pi$.

We let $N_k(s, a)$ denote the (random) number of visits of the learner to $(s,a)$ during episode $k$, so that the regret can be re-written as $R_K = \sumk \inner{N_k - q_{\optpi}}{c_k}$.
Throughout the paper, we use the notation $\inner{f}{g}$ as a shorthand for $\sums f(s)g(s)$, $\sumsa f(s,a)g(s,a)$, or $\sumh\sumsa f(s,a,h)g(s,a,h)$ when $f$ and $g$ are functions in $\fR^{\calS}$, $\fR^{\SA}$, or $\fR^{\SA\times[H]}$ (for some $H$) respectively.
Let $\calF_k$ denote the $\sigma$-algebra of events up to the beginning of episode $k$,
and $\E _k$ be a shorthand of $\E[\cdot|\calF_k]$.
%For a martingale difference sequence $X_{1:n}$ with respect to a filtration $\calF_{1:n}$ such that $\E[X_i| \calF_i]=0$, we denote by $\E_i$ a shorthand of $\E[\cdot|\calF_i]$.
%When there is no confusion,  %and $\E _k$ is a shorthand of $\E[\cdot|\calF_k]$.
%The Kullback-Leibler divergence is denoted by $\KL(p, q)=\sum_xp(x)\ln\frac{p(x)}{q(x)}$, where the summation is over the support of $p$.
For a convex function $\psi$, the Bregman divergence between $u$ and $v$ is defined as: $D_{\psi}(u, v)=\psi(u)-\psi(v)-\inner{\nabla\psi(v)}{u-v}$.
For an integer $n$, $[n]$ denotes the set $\{1, \ldots, n\}$.

\section{Minimax Regret for the Full-information Setting}\label{sec:full-known}
% !TEX root = main.tex

In this section, 
we consider the simpler full-information setting where the learner observes $c_k$ in the end of episode $k$.
Somewhat surprisingly, even in this case, ensuring optimal regret is rather challenging.
We first propose an algorithm with expected regret $\tilo{\sqrt{D\T K}}$ and a matching lower bound in \pref{subsec:full-known-expect}.
Notably, our algorithm is parameter-free and does not need to know $\T$ ahead of time.\footnote{The knowledge of $K$ is also unnecessary due to the standard doubling trick.}
Next, in \pref{subsec:full-known-prob}, by converting the problem into another loop-free SSP instance and using a skewed occupancy measure space,
we develop an algorithm that achieves the same regret bound with high probability,
although this requires the knowledge of $\T$.

\subsection{Optimal expected regret}\label{subsec:full-known-expect}
To introduce our algorithm, we first briefly review the SSP-O-REPS algorithm of~\citet{rosenberg2020adversarial}, which only achieves regret $\tilo{\frac{D}{\cmin}\sqrt{K}}$.
The idea is to run the standard Online Mirror Descent (OMD) algorithm over an appropriate occupancy measure space.
Specifically, they define the occupancy measure space parameterized by size $T > 0$ as:
\begin{equation}\label{eq:Delta_T}
\begin{split}
\Delta(T) &= \Bigg\{ q \in \fR^{\SA}_{\geq 0}: \sumsaf q(s, a)\leq T, \\
&\suma q(s, a) - \sumsaf[s', a'] P(s| s', a')q(s', a') = \Ind\{s=s_0\}, \;\forall s\in\calS  \Bigg\}.
\end{split}
\end{equation}
It is shown that every $q\in\Delta(T)$ is a valid occupancy measure induced by the policy $\pi_q$ (recall $\pi_q(a| s) \propto q(s,a)$).
Therefore, as long as $T$ is large enough such that $q_{\optpi} \in \Delta(T)$,
based on \pref{eq:regret_linear_form}, the problem is essentially translated to an instance of online linear optimization and can be solved by maintaining a sequence of occupancy measures $q_1, \ldots, q_K$ updated according to OMD:
$
q_{k+1} = \argmin_{q\in\Delta(T)}\inner{q}{c_k} + D_{\psi}(q, q_k),
$
where $\psi$ is a regularizer with the default choice being the negative entropy $\psi(q) =  \frac{1}{\eta}\sumsa q(s,a)\ln q(s,a)$ for some learning rate $\eta>0$.
See \pref{alg:worker} for the pseudocode and \citep{rosenberg2020adversarial} for the details of implementing it efficiently.

\citet{rosenberg2020adversarial} show that as long as $q_{\optpi} \in \Delta(T)$,
\pref{alg:worker} ensures $\E[R_K] = \tilo{T\sqrt{K}}$.
To ensure $q_{\optpi} \in \Delta(T)$, they set $T = \frac{D}{\cmin}$ because
$\sumsaf q_{\optpi}(s, a) = \T \leq \frac{D}{\cmin}$ as discussed in \pref{sec:prelim}.
This leads to their final regret bound $\tilo{\frac{D}{\cmin}\sqrt{K}}$.

\setcounter{AlgoLine}{0}
\begin{algorithm}[t]
\caption{SSP-O-REPS}
\label{alg:worker}
	\textbf{Input:} upper bound on expected hitting time $T$.
	
	\textbf{Define:} regularizer  $\psi(q) = \frac{1}{\eta}\sumsa q(s,a)\ln q(s,a)$ and $\eta=\min\left\{\frac{1}{2}, \sqrt{\frac{T\ln (SAT)}{DK}}\right\}$.
	
	\textbf{Initialization:} $q_1 = \argmin_{q\in\Delta(T)} \psi(q)$ where $\Delta(T)$ is defined in \pref{eq:Delta_T}.
	
	\For{$k=1,\ldots,K$}{
		Execute $\pi_{q_k}$, receive $c_k$, and update
$q_{k+1}  = \argmin_{q\in\Delta(T)}\inner{q}{c_k} + D_{\psi}(q, q_k)$.
	}
\end{algorithm}
\begin{algorithm}[t]
\caption{Adaptive SSP-O-REPS with Multi-scale Experts}
\label{alg:meta}
	\textbf{Define:} $j_0 =  \ceil{\log_2 T^{\pi^f}(s_0)}-1, b(j) = 2^{j_0+j}, 
	 \eta_j=\frac{1}{\sqrt{b(j)K\max\{D, 16\}}}, N=\ceil{\log_2 K} - j_0$.
	 
	\textbf{Define:} $\Omega=\big\{p \in \fR^N_{\geq 0} : \sum_{j=1}^N p(j)=1\big\}$
	and  $\psi(p) = \sum_{j=1}^N\frac{1}{\eta_j}p(j)\ln p(j)$.
	
	\textbf{Initialize:} $p_1 \in \Omega$ such that $p_1(j) = \frac{\eta_j}{N\eta_1}, \;\forall j \neq 1$.
	
	\textbf{Initialize:} $N$ instances of \pref{alg:worker}, where the $j$-th instance uses parameter $T=b(j)$.

	\For{$k=1,\ldots,K$}{
	     \nl For each $j \in [N]$, obtain occupancy measure $q_k^j$ from SSP-O-REPS instance $j$.
	     
		\nl Sample $j_k\sim p_k$, execute the policy induced by $q^{j_k}_k$, receive $c_k$, and feed $c_k$ to all instances. \label{line:sample}
		
		\nl Compute $\ell_k$ and $a_k$: $\ell_k(j) = \inners{q^j_k}{c_k}, a_k(j) = 4\eta_j\ell_k^2(j), \;\forall j \in [N]$.
		
		\nl Update $p_{k+1} = \argmin_{p\in\Omega}\inner{p}{\ell_k+a_k} + D_{\psi}(p, p_k)$. \label{line:OMD}
	}
\end{algorithm}

We improve their approach using the following two ideas.
First, we show a more careful analysis for the same \pref{alg:worker} and use the fact that the total expected cost of $\optpi$ is bounded by $DK$ instead of $T K$ to obtain the following stronger guarantee.
\begin{lemma}\label{lem:worker}
	If $T$ is such that $q_{\optpi}\in\Delta(T)$, then \pref{alg:worker} guarantees:
$
\E[R_K] = \tilo{\sqrt{DTK}}.
$
\end{lemma}

By using the same $T = \frac{D}{\cmin}$, this already leads to a better bound $\tilo{D\sqrt{K/\cmin}}$.
If $\T$ was known, setting $T=\T$ would also immediately give the claimed bound $\tilo{\sqrt{D\T K}}$ (since $q_{\optpi} \in \Delta(\T)$), which is optimal as we show later.

The second new idea of our approach is thus to deal with unknown $\T$ by learning it on the fly via another online learning meta-algorithm (\pref{alg:meta}).
Specifically, we maintain roughly $\ln K$ instances of \pref{alg:worker}, where the $j$-th instance sets the parameter $T$ as $b(j)$ which is roughly $2^j$, so that there always exists an instance $j^\star$ with $b(j^\star)$ very close to the unknown $\T$.
The meta-algorithm treats each instance as an expert, and in each episode, samples one of these experts and follows its policy (\pref{line:sample}).
If the regret of this meta-algorithm to instance $j^\star$ is no larger than $\tilo{\sqrt{D\T K}}$, then the overall regret to $\optpi$ would clearly also be $\tilo{\sqrt{D\T K}}$.

While seemingly this appears to be a classic expert problem and might be solved by the standard Hedge algorithm~\citep{freund1997decision},
the key challenge is that the loss for each expert $j$, denoted by $\ell_k(j) = \inners{q^j_k}{c_k}$ (for episode $k$), has a different scale.
Indeed, we have $\ell_k(j) \leq \sumsa q^j_k(s,a) \leq b(j)$.
Standard algorithms such as Hedge have a regret bound that depends on a uniform upper bound of all losses as large as $b(N) \approx K$ in our case, leading to a vacuous bound.
More advanced ``multi-scale'' algorithms~\citep{bubeck2017online,foster2017parameter,cutkosky2018black} mitigate the issue and ensure regret $\tilo{b(j^\star)\sqrt{K}}$ comparing to expert $j^\star$,
which still leads to $\tilo{\T\sqrt{K}}$ regret overhead and ruins the final bound.

To address this challenge, we propose a new multi-scale expert algorithm with regret bound $\otil\Big(\sqrt{b(j^\star)\E[\sumk \ell_k(j^\star)]}\Big)$, which is always no worse than previous works since $\sumk \ell_k(j^\star) \leq b(j^\star) K$.
The algorithm is similar to that of~\citep{bubeck2017online} which is OMD over the $(N-1)$-dimensional simplex with a weighted negative entropy regularizer $\psi(p) = \sum_{j=1}^N\frac{1}{\eta_j}p(j)\ln p(j)$.
Here, each expert uses a different learning rate $\eta_j$ that depends on the corresponding scale $b(j)$.
The key difference of our algorithm is that we also add a correction term $a_k(j) = 4\eta_j\ell_k^2(j)$ to the loss $\ell_k(j)$ (\pref{line:OMD}), an idea used in previous works such as~\citep{steinhardt2014adaptivity, wei2018more} to obtain a bound in terms of the loss of the benchmark $\sumk \ell_k(j^\star)$.
Another important tweak is to set the initial distribution for expert $j\neq 1$ to be $\frac{\eta_j}{N\eta_1}$.
We note that this new and improved multi-scale expert algorithm might be of independent interest.

To see why this improved bound $\otil\Big(\sqrt{b(j^\star)\E[\sumk \ell_k(j^\star)]}\Big)$ helps, note that \pref{lem:worker} imples: $\E[\sumk \ell_k(j^\star)] \leq \E[\sumk J_k^{\optpi}(s_0)] + \tilo{\sqrt{D\T K}} \leq DK + \tilo{\sqrt{D\T K}}$.
Thus, the overhead of the meta-algorithm is of order $\tilo{\sqrt{D\T K}}$ as desired. We summarize the final guarantee below.

\begin{theorem}\label{thm:meta}
	\pref{alg:meta} enjoys the following expected regret bound:
$
		\E [R_K] = \tilO{\sqrt{D\T K}}.
$
\end{theorem}

\paragraph{Lower bound}
Our regret bound stated in \pref{thm:meta} not only improves that of~\citep{rosenberg2020adversarial}, but is also optimal up to logarithmic terms as shown in the following lower bound.

\begin{theorem}\label{thm:full-known-lb}
	For any $D, \T, K$ with $K \geq \T\geq D+1$, there exists an SSP instance such that its diameter is $D+2$, the optimal policy has hitting time $\T+1$, and the expected regret of any learner after $K$ episodes is at least $\lowO{\sqrt{D\T K}}$ under the full-information and known transition setting.
\end{theorem}

Similarly to most lower bound proofs, our proof also constructs an environment with stochastic costs and with a slightly better state hidden among other equally good states, and argues that the expected regret of any learner with respect to the randomness of the environment has to be $\lowO{\sqrt{D\T K}}$.
At first glance, this appears to be a contradiction to existing results for SSP with stochastic costs~\citep{tarbouriech2019no,cohen2020near}, where the optimal regret is independent of $\T$.
However, the catch is that ``stochastic costs'' has a different meaning in these works.
Specifically, it refers to a setting where the cost for each state-action pair is drawn independently from a fixed distribution {\it every time} it is visited, and is revealed to the learner immediately. % even before the episode ends.
On the other hand, ``stochastic costs'' in our lower bound proof refers to a setting where at the beginning of each episode $k$, $c_k$ is sampled once from a fixed distribution and then {\it fixed} throughout the episode.
Moreover, it is revealed only after the episode ends.
It can be shown that our setting is harder due to the larger variance of costs, explaining our larger lower bound and the seemingly contradiction.

\subsection{Optimal high-probability regret}\label{subsec:full-known-prob}
To obtain a high-probability regret bound, one needs to control the deviation between
the actual total cost of the learner $\sumk \inner{N_k}{c_k}$ and its expectation $\sumk \inner{q_k}{c_k}$.
While for most online learning problems with full information, similar deviation can be easily controlled by the Azuma's inequality,
this is not true for SSP as pointed out in~\citep{rosenberg2020adversarial}, due to the lack of an almost sure upper bound on the random variable $\inner{N_k}{c_k}$.
\citet[Lemma~E.1]{rosenberg2020adversarial} point out that with high probability $\sumsa N_{\optpi}(s,a)$ is bounded by $\Tmax$,
and thus it is natural to enforce the same for $N_k$.
However, this at best leads to a bound of order $\tilo{\sqrt{D\Tmax K}}$.
To achieve the optimal regret,
we start with a closer look at the variance of the actual cost of any policy,
showing that it is in fact related to the corresponding cost-to-go function.

\begin{lemma}\label{lem:deviation}
Consider executing a stationary policy $\pi$ in episode $k$.
Then $\E[\inner{N_{k}}{c_k}^2]\leq 2\inner{q_{\pi}}{J^{\pi}_k}$.
\end{lemma}

For the optimal policy $\optpi$, although $\inner{q_{\optpi}}{J^{\optpi}_k}$ can still be as large as $\Tmax$,
one key observation is that the sum of these quantities over $K$ episodes is at most $D\T K$ since $\sumk\inner{q_{\optpi}}{J^{\optpi}_k} = \sums q_{\optpi}(s)\sumk J^{\optpi}_k(s) \leq DK  \sums q_{\optpi}(s) = D\T K$,
where the inequality is again due to the optimality of $\optpi$ and the existence of the fast policy $\pi^f$: $\sumk J^{\optpi}_k(s) \leq \sumk J^{\pi^f}_k(s) \leq  DK$.
Given this observation, it is tempting to enforce that the learner's policies $\pi_1, \ldots, \pi_K$ are also such that $\sumk\inner{q_{\pi_k}}{J^{\pi_k}_k} \leq D\T K$,
which would be enough to control the deviation between $\sumk \inner{N_{k}}{c_k}$ and $\sumk \inner{q_{\pi_k}}{c_k}$ by $\tilo{\sqrt{D\T K}}$ as desired by Freedman's inequality.
However, it is unclear how to enforce this constraint since it depends on all the cost functions unknown ahead of time.
In fact, even if the cost functions were known, the constraint is also non-convex due to the complicated dependence of $J^{\pi}_k$ on $q_{\pi}$.
To address these issues, we propose two novel ideas.

\paragraph{First idea: a loop-free reduction}
Our first idea is to reduce the problem to a {\it loop-free} MDP so that the variance $\E[\inner{N_{k}}{c_k}^2]$ takes a much simpler form that is linear in both the occupancy measure and the cost function.
Moreover, the reduction only introduces a small bias in the regret between the original problem and its loop-free version.
The construction of the loop-free MDP is basically to duplicate each state by attaching a time step $h$ for $H_1$ steps, and then connect all states to some virtual {\it fast} state that lasts for another $H_2$ steps.
Formally, we define the following.
\begin{definition}\label{def:reduction}
	For an SSP instance $M=(\calS, s_0, g, \calA, P)$ with cost functions $c_{1:K}$,
	we define, for horizon parameters $H_1, H_2 \in \fN$, another loop-free SSP instance $\tilM=(\tilS, \tils_0, g, \tilA, \tilP)$ with cost function $\tilc_{1:K}$ as follows: 
	\begin{itemize}
	\item $\tilS=(\calS\cup\{s_f\})\times [H]$ where $s_f$ is an artificially added ``fast'' state and $H = H_1 + H_2$;
	\item $\tils_0 = (s_0, 1)$ and the goal state $g$ remains the same;
	\item $\tilA = \calA\cup\{a_f\}$, where $a_f$ is an artificially added action that is only available at $(s_f, h)$ for $h\in[H]$ (the available action set at $(s,h)$ is $\calA_s$ for all $s\neq s_f$ and $h\in [H]$);
	\item transition from $(s,h)$ to $(s', h')$ is only possible when $h'=h+1$:
	for the first $H_1$ layers, the transition follows the original MDP in the sense that $\tilP((s', h+1)|(s,h),a) = P(s'|s,a)$ and $\tilP(g|(s,h),a) = P(g|s,a)$ for all $h < H_1$ and $(s,a)\in\SA$;
	from layer $H_1$ to layer $H$, all states transit to the fast state: $\tilP((s_f, h+1)|(s,h),a) = 1$ for all $H_1 \leq h < H$ and $(s,a)\in \tilSA \triangleq \SA \cup \{(s_f, a_f)\}$;
	finally, the last layer transits to the goal state always: $\tilP(g|(s,H),a)=1$ for all $(s,a) \in \tilSA$;
	\item 
	cost function is such that $\tilc_k((s,h),a) = c_k(s,a)$ and $\tilc_k((s_f,h),a_f) = 1$ for all $(s,a)\in\SA$ and $h\in[H]$; for notational convenience, we also write $\tilc_k((s,h),a)$ as $c_k(s,a,h)$.
	\end{itemize}
\end{definition}

Note that in this definition, there are some redundant states such as $(s,h)$ for $s\in\calS$ and $h>H_1$ or $(s_f, h)$ for $h \leq H_1$ since they will never be visited.
However, having these redundant states greatly simplifies our presentation.
For notations related to the loop-free version, we often use a tilde symbol to distinguish them from the original counterparts (such as $\tilM$ and $\tilS$),
and for a function $\tilf((s,h),a)$ that takes a state in $\tilM$ and an action as inputs,
we often simplify it as $f(s, a, h)$ (such as $c_k$ and $q_k$).
For such a function, we will also use the notation $\h{f} \in \fR^{\tilSA\times[H]}$ such that $(\h{f})(s, a, h)=h\cdot f(s, a, h)$.
Similarly, for a function $f \in \fR^{\tilSA}$, we use the same notation $\h{f} \in \fR^{\tilSA\times[H]}$ such that $(\h{f})(s, a, h)=h\cdot f(s, a)$.

As mentioned, one key reason of considering such a loop-free MDP is that the variance of the learner's actual cost takes a much simpler form that is linear in both the occupancy measure and the cost function, as shown in the lemma below (which is an analogue of \pref{lem:deviation}).

\begin{lemma}\label{lem:deviation_loop_free}
Consider executing a stationary policy $\tilpi$ in $\tilM$ in episode $k$
and let $\tilN_{k}(s, a, h) \in \{0,1\}$ denote the number of visits to state-action pair $((s, h),a)$.
Then 
	$\E[\inners{\tilN_{k}}{c_k}^2]\leq 2\inner{q_\tilpi}{\h{c_k}}$.
\end{lemma}

Next, we complete the reduction by describing how one can solve the original problem via  solving its loop-free version.
Given a policy $\tilpi$ for $\tilM$,
we define a {\it non-stationary} policy $\sigma(\tilpi)$ for $M$ as follows: for each step $h \leq H_1$, follow $\tilpi(\cdot|(s,h))$ when at state $s$;
after the first $H_1$ steps (if not reaching $g$ yet), execute the fast policy $\pi^f$ until reaching the goal state $g$.
When executing $\sigma(\tilpi)$ in $M$ for episode $k$, we overload the notation $\tilN_{k}$ defined in \pref{lem:deviation_loop_free} and let $\tilN_{k}(s, a, h)$ be $1$ if $(s,a)$ is visited at time step $h \leq H_1$, or $0$ otherwise; and $\tilN_{k}(s_f, a_f, h)$ be $1$ if $H_1 < h \leq H$ and the goal state $g$ is not reached within $H_1$ steps, or 0 otherwise.
Clearly, $\tilN_{k}(s, a, h)$ indeed follows the same distribution as the number of visits to state-action pair $((s, h),a)$ when executing $\tilpi$ in $\tilM$. 
We also define a deterministic policy $\tiloptpi$ for $\tilM$ that mimics the behavior of $\optpi$ in the sense that $\tiloptpi(s,h) = \optpi(s)$ for $s\in\calS$ and $h\leq H_1$ (for larger $h$, $s$ has to be $s_f$ and the only available action is $a_f$).
The next lemma shows that,
as long as the horizon parameters $H_1$ and $H_2$ are set appropriately,
this reduction makes sure that the regret between these two problems are similar.

\begin{lemma}\label{lem:loop-free}
Suppose $H_1 \geq 8\Tmax\ln K, H_2 = \ceil{4D\ln\frac{4K}{\delta}}$ and $K\geq D$ for some %confidence level 
$\delta \in (0,1)$.
Let $\tilpi_1, \ldots, \tilpi_K$ be policies for $\tilM$ with occupancy measures $q_1, \ldots, q_K \in [0,1]^{\tilSA\times[H]}$.
Then the regret of executing $\sigma(\tilpi_1), \ldots, \sigma(\tilpi_K)$ in $M$ satisfies: 1) for any $\lambda \in (0, 2/H]$, with probability $1-\delta$,
	\begin{align*}
	R_K \leq \sumk\inner{\tilN_{k}-\tiloptq}{c_k} + \tilO{1} \leq \underbrace{\sumk\inner{q_k-q_{\tiloptpi}}{c_k}}_{\reg} +  \lambda\underbrace{\sumk\inner{q_k}{\h{c_k}}}_{\var} + \frac{2\ln\left(\nicefrac{2}{\delta}\right)}{\lambda} + \tilO{1},
	\end{align*}
and 2) $\E[R_K] \leq \E[\reg] + \tilO{1}$.
\end{lemma}

Note that the $\reg$ term is the expected regret (to $\tiloptpi$) in $\tilM$ and can again be controlled by OMD.
The $\var$ term comes from the derivation between the actual cost of the learner in $\tilM$ and its expectation, according to Freedman's inequality and \pref{lem:deviation_loop_free}.
At this point, one might wonder whether directly applying an existing algorithm such as~\citep{zimin2013online} for loop-free MDPs solves the problem, since \pref{lem:loop-free} shows that the regret in these two problems are close.
Doing so, however, leads to a suboptimal bound of order $\tilo{H\sqrt{K}} = \tilo{(\Tmax+D)\sqrt{K}}$.
This is basically the same as trivially bounding $\var$ by $H^2K$.
It is thus critical to better control this term using properties of the original problem, which requires the second idea described below.

\paragraph{Second idea: skewed occupancy measure space}
Similarly to earlier discussions, it can be shown that $\sumk\inner{q_{\tiloptpi}}{\h{c_k}}= \bigO{D\T K}$ (\pref{lem:tiloptpi_variance}), making it hopeful to bound $\var$ by the same.
However, even though the variance now takes a simpler form, it is still unclear how to directly enforce the algorithm to satisfy $\var=\bigO{D\T K}$.
Instead, we take a different route and make sure that the $\reg$ term is at most $\tilO{\sqrt{D\T K}+\lambda D\T K} - \lambda\var$, thus canceling the variance term.
To do so, thanks to the simple form of $\var$, it suffices to inject a small positive bias into the action space of OMD, making it a {\it skewed} occupancy measure space: 
$\Omega = \{\qfeat=q + \lambda \h{q} : q\in \tilDelta(\T)\}$
where $\tilDelta(\T)$ is the counterpart of $\Delta(\T)$ for $\tilM$ (see \pref{eq:tilDelta_T} in \pref{app:full-known} for the spelled out definition).
Indeed, by similar arguments from \pref{subsec:full-known-expect}, operating OMD over this space ensures a bound of order $\bigO{\sqrt{D\T K}}$ on the ``skewed regret'':
$
\sumk\inner{(q_k +  \lambda \h{q_k})-(q_{\tiloptpi}+ \lambda \h{q_{\tiloptpi}})}{c_k}
= \reg + \lambda \var - \lambda \sumk\inner{q_{\tiloptpi}}{\h{c_k}},
$
and we already know that the last term is of order $\bigO{\lambda D\T K}$.
Rearranging thus proves the desired bound on $\reg$,
and finally picking the optimal $\lambda$ to trade off the term $\frac{2\ln(\nicefrac{2}{\delta})}{\lambda}$ leads to the optimal bound.
We summarize the final algorithm in \pref{alg:full-known-hp} and its regret guarantee below.
(Note that the algorithm can be implemented efficiently since $\Omega$ is a convex polytope with $\bigo{SAH}$ constraints.)

\begin{algorithm}[t]
\caption{SSP-O-REPS with Loop-free Reduction and Skewed Occupancy Measure}
\label{alg:full-known-hp}
	\textbf{Input:} Upper bound on expected hitting time $T$, horizon parameter $H_1$, confidence level $\delta$
	
	\textbf{Parameters:} $\eta=\min\Big\{\frac{1}{2}, \sqrt{\frac{T}{DK}}\Big\}, \lambda=\sqrt{\frac{\ln(\nicefrac{1}{\delta})}{DTK}}, H_2=\ceil{4D\ln\frac{4K}{\delta}}$
	
	\textbf{Define:} $H=H_1+H_2$, regularizer $\regz(\qfeat) = \frac{1}{\eta}\sum_{h=1}^{H}\sum_{(s,a)\in\tilSA} \qfeat(s, a, h)\ln \qfeat(s, a, h)$ 
	
	\textbf{Define:} decision set $\Omega = \{\qfeat=q + \lambda \h{q} : q\in \tilDelta(T)\}$ (with $\tilDelta(T)$ defined in \pref{eq:tilDelta_T}) 
	
	\textbf{Initialization:} $\qfeat_1 = q_{1}+ \lambda\h{q_{1}} = \argmin_{\qfeat\in\Omega} \regz(\qfeat)$.
	
	\For{$k=1,\ldots,K$}{
		Execute $\sigma(\tilpi_k)$ where $\tilpi_k$ is such that $\tilpi_k(a|(s, h))\propto q_k(s, a, h)$, and receive $c_k$.
		
		Update $\qfeat_{k+1} = q_{k+1}+ \lambda\h{q_{k+1}} = \argmin_{\qfeat\in\Omega}\inner{\qfeat}{c_k} + D_{\regz}(\qfeat, \qfeat_k)$.
	}
\end{algorithm}

\begin{theorem}\label{thm:full-known-hp}
	If $T\geq \T+1$, $H_1\geq 8\Tmax\ln K$, and $K \geq H^2\ln\rbr{\frac{1}{\delta}}$, then with probability at least $1-\delta$, \pref{alg:full-known-hp} ensures $R_K=\tilo{\sqrt{DT K\ln\left(\nicefrac{1}{\delta}\right)}}$.
\end{theorem}

To obtain the optimal bound, we need to set $T = c\T+1$ for any constant $c\geq 1$.
Unfortunately, we are unable to extend the idea from \pref{subsec:full-known-expect} to also learn this parameter on the fly, and we leave it as an important future direction (see \pref{sec:conclusion} for more discussions).
Note that, however, in the construction of the lower bound (\pref{thm:full-known-lb}), $\T$ is indeed known to the learner.
Thus, assuming the knowledge of $\T$ {\it does not make the problem any easier information-theoretically}.
As for the parameter $H_1$, we can always set it to something large such as $K^{1/3}$ so that the conditions of the theorem hold for large enough $K$ (though leading to a larger time complexity of the algorithm).

%since our bound only needs $K$ to be large enough,
%we can always set it to be a small order term w.r.t. $K$ (such as $K^{1/3}$) to ensure a small additional cost even if the condition of the theorem does not hold (though leading to a larger time complexity of the algorithm).

We also remark that instead of injecting bias to the occupancy measure space,
one can obtain the same by injecting a similar positive bias to the cost function.
However, we use the former approach because it turns out to be critical for the bandit feedback setting that we consider in the next section.

\section{Minimax Regret for the Bandit Feedback Setting}\label{sec:bandit-known}
% !TEX root = main.tex

We now consider the more challenging case with bandit feedback, that is,
at the end of each episode, the learner only receives the cost of the visited state-action pairs.
A standard technique in the adversarial bandit literature is to construct an importance-weighted cost estimator $\hatc_k$ for $c_k$ and then feed it to OMD,
which is even applicable to learning loop-free SSP~\citep{zimin2013online, jin2019learning, lee2020bias}.
For general SSP, the natural importance-weighted estimator $\hatc_k$ is: $\hatc_k(s, a)=\frac{N_{k}(s, a)c_k(s, a)}{q_{k}(s, a)}$ where $N_k(s, a)$ is the number of visits to $(s,a)$ and $q_k$ is the occupancy measure of the policy executed in episode $k$.
This is clearly unbiased since $\E_k[N_{k}(s, a)] = q_{k}(s, a)$.

However, it is well-known that unbiasedness alone is not enough --- the variance of the estimator also plays a key role in the OMD analysis even if one only cares about expected regret.
For example, if we still use the entropy regularizer as in \pref{sec:full-known},
the so-called stability term of OMD is in terms of the weighted variance $\sumsa q_{k}(s, a)\E_k[\hatc_k^2(s, a)]=\sumsa \frac{\E_k[N^2_{k}(s, a)]c_k^2(s, a)}{q_{k}(s, a)}$.
While this term is nicely bounded in the loop-free case (since $N_{k}(s, a)$ is binary and thus $\E_k[N^2_{k}(s, a)] = q_k(s,a)$ cancels out the denominator),
unfortunately it can be prohibitively large in the general case.
In light of this, it might be tempting to use our loop-free reduction again and then directly apply an existing algorithm such as~\citep{zimin2013online}.
However, this again leads to a suboptimal bound with dependence on $H = \tilO{\Tmax}$.
It turns out that this is significantly more challenging than other bandit problems and requires a combination of various techniques, as described below.

\paragraph{Log-barrier regularizer}
Although the entropy regularizer is a classic choice for OMD to deal with bandit problems, in recent years, a line of research discovers various advantages of using a different regularizer called {\it log-barrier} (see e.g.~\citep{FosterLiLySrTa16, agarwal2017corralling, wei2018more, luo2018efficient, bubeck2019improved, kotlowski2019bandit, LLZ20}). 
In our context, the log-barrier regularizer is $-\frac{1}{\eta}\sumsa \ln q(s,a)$,
and it indeed leads to a smaller stability term in terms of $\sumsa q_{k}^2(s, a)\E_k[\hatc_k^2(s, a)]=\sumsa \E_k[N^2_{k}(s, a)]c_k^2(s, a)$ (note the extra $q_{k}(s, a)$ factor compared to the case of entropy).
This term is further bounded by $\E_k[\inner{N_k}{c_k}^2]$, which is exactly the variance of the learner's actual cost considered in \pref{subsec:full-known-prob}!

\paragraph{Loop-free reduction and skewed occupancy measure}
Based on the observation above, it is natural to apply the same ideas of loop-free reduction and skewed occupancy measure from \pref{subsec:full-known-prob} to deal with the stability term $\E_k[\inner{N_k}{c_k}^2]$.
However, some extra care is needed when using log-barrier in the loop-free instance $\tilM$.
Indeed, directly using $\regz(\qfeat) = -\frac{1}{\eta}\sum_h\sumsa \ln \qfeat(s,a,h)$ would lead to another term of order $\tilo{HSA/\eta}$ in the OMD analysis and ruin the bound.
Instead, taking advantage of the fact that $c_k(s,a,h)$ is the same for a fixed $(s,a)$ pair regardless of the value of $h$,\footnote{%
This also explains why injecting the bias to the occupancy space instead of the cost vectors is important here, as mentioned in the end of \pref{sec:full-known}, since the latter makes the cost different for different $h$.
}
we propose to perform OMD with $\qfeat(s,a) = \sum_h \qfeat(s,a,h)$ for all $(s,a)\in\tilSA$ as the variables, even though the skewed occupancy measure $\Omega$ is still defined in terms of $\qfeat(s,a,h)$ as in \pref{alg:full-known-hp}.
More specifically, this means that our regularizer is $\regz(\qfeat) = -\frac{1}{\eta}\sumsa \ln \qfeat(s,a)$,
and the cost estimator is $\hatc_k(s,a) = \frac{\tilN_{k}(s, a)c_k(s, a)}{q_{k}(s, a)}$
where $\tilN_{k}(s, a) = \sum_h \tilN_{k}(s, a, h)$ and $q_k(s,a) = \sum_h q_k(s,a,h)$.
This completely avoids the factor $H$ in the analysis (other than lower order terms).

With the ideas above, we can already show an optimal expected regret bound for an {\it oblivious} adversary who selects $c_k$ independent of the learner's randomness.
We summarize the algorithm in \pref{alg:bandit-known-pseudo} and its guarantee in the following theorem.

\begin{algorithm}[t]
\caption{Log-barrier Policy Search for SSP}
\label{alg:bandit-known-pseudo}
	\textbf{Input:} Upper bound on expected hitting time $T$ and horizon parameter $H_1$.
	
	\textbf{Parameters:} $\eta= \sqrt{\frac{SA}{DTK}}, \lambda=8\eta, H_2=\ceil{4D\ln\frac{4K}{\delta}}$, $H=H_1+H_2$
	
	\textbf{Define:} regularizer $\regz(\qfeat) = -\frac{1}{\eta}\sum_{(s,a)\in\tilSA} \ln \qfeat(s, a)$ where $\qfeat(s, a) = \sumh\qfeat(s, a, h)$
	
	\textbf{Define:} decision set $\Omega = \{\qfeat=q + \lambda \h{q} : q\in \tilDelta(T)\}$ (with $\tilDelta(T)$ defined in \pref{eq:tilDelta_T}) 
	
	\textbf{Initialization:} $\qfeat_1 = q_{1}+ \lambda\h{q_{1}} = \argmin_{\qfeat\in\Omega} \regz(\qfeat)$.
	
	\For{$k=1,\ldots, K$}{
		Execute $\sigma(\tilpi_k)$ where $\tilpi_k$ is such that $\tilpi_k(a| (s, h)) \propto q_k(s, a, h)$.
		
		Construct cost estimator $\hatc_k \in \fR_{\geq 0}^{\tilSA}$ such that $\hatc_k(s,a) = \frac{\tilN_{k}(s, a)c_k(s, a)}{q_{k}(s, a)}$
where $\tilN_{k}(s, a) = \sum_h \tilN_{k}(s, a, h)$ and $q_k(s,a) = \sum_h q_k(s,a,h)$
($\tilN_{k}$ is defined after \pref{lem:deviation_loop_free}).

		Update $
			\qfeat_{k+1} = q_{k+1}+ \lambda\h{q_{k+1}} = \argmin_{\qfeat\in\Omega}\sumsa  \qfeat(s, a)\hatc_k(s, a) + D_{\psi}(\qfeat, \qfeat_k)$.
	}
\end{algorithm}

\begin{theorem}\label{thm:bandit-known-pseudo}
If $T\geq \T+1$, $H_1\geq 8\Tmax\ln K$, and $K \geq 64SAH^2$, then \pref{alg:bandit-known-pseudo} ensures $\E\sbr{R_K} = \tilO{\sqrt{DTSAK}}$ for an oblivious adversary.
\end{theorem}

Setting $T=\T+1$ leads to $\tilo{\sqrt{D\T SAK}}$, which is optimal in light of the following lower bound theorem
(the adversary is indeed oblivious in the lower bound construction).

\begin{theorem}\label{thm:bandit-known-lb}
For any $D, \T, K, S\geq 4$ with $K \geq S\T$ and $\T\geq D+1$,
there exists an SSP problem instance with $S$ states and $A=\bigo{1}$ actions such that its diameter is $D+2$, the optimal policy has expecting hitting time $\T+1$, and the expected regret of any learner after $K$ episodes is at least $\lowO{\sqrt{D\T SAK}}$ under the bandit feedback and known transition setting.
\end{theorem}

To further obtain a high probability regret bound for general adaptive adversaries (thus also a more general expected regret bound), it is important to analyze the the derivation between the optimal policy's estimated total loss $\sum_k \inner{q_{\tiloptpi}}{\hatc_k}$ and its expectation $\sum_k \inner{q_{\tiloptpi}}{c_k}$.
Using Freedman's inequality, 
we need to carefully control the conditional variance 
$\E_k[\hatc_k^2(s,a)] = \frac{\E_k[\tilN^2_{k}(s, a)]c_k^2(s, a)}{q_k^2(s,a)}$ for each $(s,a)$,
which is much more difficult than the aforementioned stability term due to the lack of the extra $q_k^2(s,a)$ factor.
To address this, we first utilize the simpler form of  $\E_k[\tilN^2_{k}(s, a)]$ in the loop-free setting and bound it by $\sum_h h q_k(s,a,h)$ (see \pref{lem:simpler_form_N_times_c}).
Then, with $\rho_K(s,a) = \max_{k}\frac{1}{q_k(s,a)}$ and $b_k(s,a) = \frac{\sum_h h q_k(s,a,h)c_k(s,a)}{q_k(s,a)}$, we bound the key term in the derivation $\sum_k \inner{q_{\tiloptpi}}{\hatc_k - c_k}$ by
\[
\sumsa q_{\tiloptpi}(s,a)\sqrt{\rho_K(s,a)\sumk b_k(s,a)}
\leq \frac{1}{\eta}\inner{q_{\tiloptpi}}{\rho_K} + \eta\sumk \inner{q_{\tiloptpi}}{b_k},
\]
where $q_{\tiloptpi}(s,a) = \sum_h q_{\tiloptpi}(s,a,h)$ and the last step is by AM-GM inequality (see \pref{lem:DEV_1_2} for details).
The last two terms above are then handled by the following two ideas respectively.

\paragraph{Increasing learning rate}
The first term $\frac{1}{\eta}\inner{q_{\tiloptpi}}{\rho_K}$ appears in the work of~\citep{lee2020bias} already for loop-free MDPs and can be canceled by a negative term introduced by an increasing learning rate schedule. 
(See the last for loop of \pref{alg:bandit-known-hp} and \pref{lem:OMD-increasing-eta}.)

\paragraph{Injecting negative bias to the costs}
To handle the second term $\eta\sumk \inner{q_{\tiloptpi}}{b_k}$, note again that its counterpart $\eta\sumk \inner{q_{k}}{b_k}$ is exactly $\eta \sumk \inner{q_k}{\h{c_k}}$, a term that can be canceled by the skewed occupancy measure as discussed.
Therefore, if we could inject another negative bias term into the cost vectors, that is, replacing $\hatc_k$ with $\hatc_k - \eta b_k$, then this bias would cancel the term $\eta\sumk \inner{q_{\tiloptpi}}{b_k}$ while introducing the term $\eta\sumk \inner{q_{k}}{b_k}$ that could be further canceled by the skewed occupancy measure.
However, the issue is that $b_k$ depends on the unknown true cost $c_k$.
We address this by using $\hatb_k$ instead which replaces $c_k$ with $\hatc_k$, that is, $\hatb_k(s,a) = \frac{\sum_h h q_k(s,a,h)\hatc_k(s,a)}{q_k(s,a)}$.
This leads to yet another derivation term between $\hatb_k$ and $b_k$ that needs to be controlled in the analysis.
Fortunately, this term is of lower order compared to others since it is multiplied by $\eta$ (see \pref{lem:DEV_3_4}).
Note that at this point we have used both the positive bias from the skewed occupancy measure space and the negative bias from the cost estimators, which we find intriguing.

Combining everything, our final algorithm is summarized in \pref{alg:bandit-known-hp} (see \pref{app:bandit-known} due to space limit).
The following theorem shows that, with the knowledge of $\T$ or a suitable upper bound, our algorithm again achieves the optimal regret bound with high probability.

\begin{theorem}\label{thm:bandit-known-hp}
If $T\geq \T+1$, $H_1\geq 8\Tmax\ln K$, and $K$ is large enough ($K\gtrsim SAH^2\ln\rbr{\frac{1}{\delta}}$), then \pref{alg:bandit-known-hp} ensures $R_K = \tilO{\sqrt{DTSAK\ln\left(\nicefrac{1}{\delta}\right)}}$ with probability at least $1-6\delta$.
\end{theorem}

\section{Conclusion}\label{sec:conclusion}
In this paper, we develop matching upper and lower bounds for the stochastic shortest path problem with adversarial costs and unknown transition, significantly improving previous results.
Our algorithms are built on top of a variety of techniques that might be of independent interest.

There are two key future directions.
The first one is to develop parameter-free and optimal algorithms without the knowledge of $\T$. 
We only achieve this in the full-information setting for expected regret bounds.
Indeed, generalizing our techniques that learn $\T$ automatically to obtain a high-probability bound in the full-information setting boils down to getting the same multi-scale expert result with high probability, which is still open unfortunately (see also discussions in~\citep[Section~5]{chen2021impossible}).
The difficulty lies in bounding the deviation between the learner's expected loss and the actual loss in terms of the loss of the unknown comparator.
On the other hand, it is also difficult to generalize our technique to obtain an expected bound in the bandit setting (without knowing $\T$), since this becomes a bandit-of-bandits type of framework and is known to suffer some tuning issues; see for example~\citep[Appendix A.2]{foster2019model}.

The second future direction is to figure out the minimax regret of the more challenging setting where the transition is unknown.
We note that our loop-free reduction is readily to be applied to this case, but due to some technical challenges, it is highly unclear how to avoid having the dependence on $\Tmax$ in the regret bounds.
A follow-up work by the first two authors~\citep{chen2021finding} makes some progress in this direction, but the minimax regret remains unknown in this case.

\acks{
The authors thank Tiancheng Jin for many helpful discussions. 
This work is supported by NSF Award IIS-1943607 and a Google Faculty Research Award.
}

\bibliography{bib}
\newpage

\appendix

% !TEX root = main.tex

\section{Omitted details for \pref{sec:full-known}}\label{app:full-known}
% !TEX root = main.tex

In this section, we provide all proofs for \pref{sec:full-known}.

\subsection{\pfref{lem:worker}}

\begin{proof}
	By standard OMD analysis (see for example Eq.~(12) of~\citep{rosenberg2020adversarial}), for any $q\in\Delta(T)$ we have:
	\begin{align}
		\sumk\inner{q_k-q}{c_k} \leq D_{\psi}(q, q_1) + \sumk\inner{q_k-q'_{k+1}}{c_k}, \label{eq:standard_OMD}
	\end{align}
	where $q'_{k+1}=\argmin_{q\in\fR^{\SA}}\inner{q}{c_k} + D_{\psi}(q, q_k)$, or equivalently, with the particular choice of the regularizer, $q'_{k+1}(s,a) = q_k(s,a) e^{-\eta c_k(s,a)}$. Applying the inequality $1 - e^{-x} \leq x$, we obtain
	\begin{align*}
		\sumk\inner{q_k-q'_{k+1}}{c_k} \leq \eta\sumk\sumsa q_k(s, a)c_k^2(s, a) \leq \eta\sumk\inner{q_k}{c_k}.
	\end{align*}
	Substituting this back into \pref{eq:standard_OMD}, choosing $q = q_{\optpi}$ (recall the condition $q_{\optpi} \in \Delta(T)$ of the lemma), and rearranging, we arrive at
	\begin{align}
		\sumk\inner{q_k - q_{\optpi}}{c_k} &\leq \frac{1}{1-\eta}\left( D_{\psi}(q_{\optpi}, q_1)+\eta\sumk\inner{q_{\optpi}}{c_k} \right) \notag\\
		&\leq 2D_{\psi}(q_{\optpi}, q_1)+2\eta\sumk\inner{q_{\optpi}}{c_k}. \label{eq:OMD_small_loss}
	\end{align}
	It remains to bound the last two terms. For the first one, 
	since $q_1$ minimizes $\psi$ over $\Delta(T)$, we have $\inner{\nabla\psi(q_1)}{q_{\optpi}-q_1}\geq 0$, and thus
	\begin{align*}
		D_{\psi}(q_{\optpi}, q_1) \leq \psi(q_{\optpi}) - \psi(q_1) 
		&= \frac{1}{\eta}\sumsa q_{\optpi}(s,a) \ln q_{\optpi}(s,a) - \frac{1}{\eta}\sumsa q_1(s,a) \ln q_1(s,a) \\
		&\leq \frac{1}{\eta}\sumsa q_{\optpi}(s,a) \ln T - \frac{T}{\eta}\sumsa \frac{q_1(s,a)}{T} \ln \frac{q_1(s,a)}{T} \\
		&\leq \frac{T\ln(T)}{\eta} + \frac{T\ln(SA)}{\eta}
		= \frac{T\ln(SAT)}{\eta}.
	\end{align*}
	For the second one, we use the fact $\sumk\inner{q_{\optpi}}{c_k}\leq\sumk\inner{q_{\pi^f}}{c_k}\leq DK$.
	Put together, this implies
	\begin{align*}
		\sumk\inner{q_k - q_{\optpi}}{c_k} 
		&\leq \frac{2T\ln (SAT)}{\eta} + 2\eta DK.
	\end{align*}
	With the optimal $\eta=\min\left\{\frac{1}{2}, \sqrt{\frac{T\ln (SAT)}{DK}}\right\}$, we have thus shown
	\begin{align*}
		\E[R_k] = \E\left[\sumk\inner{q_k - q_{\optpi}}{c_k}\right] = \bigO{ \sqrt{DTK\ln (SAT)} + T\ln (SAT) } = \tilO{\sqrt{DTK}},
	\end{align*}
	completing the proof.
\end{proof}

\subsection{\pfref{thm:meta}}
\begin{proof}
     First, note that the value of $j_0$ is such that the smallest parameter $b(1)$ is larger than $T^{\pi^f}(s_0)$ and thus $\Delta(b(j))$ is non-empty for all $j\in [N]$, making all $N$ instances of \pref{alg:worker} well-defined.
     Next, let $j^\star$ be the index of the instance with size parameter closest to the unknown parameter $\T$, that is, $\frac{b(j^\star)}{2} \leq \T \leq b(j^\star)$.
     Such $j^\star$ must exist since $b(N) \geq K$ and we only need to consider the case $\T \leq K$ (otherwise the claimed regret bound is vacuous).
     Now we decompose the regret as two parts, the regret of the meta algorithm to instance $j^\star$, and the regret of instance $j^\star$ to the best policy:
	\begin{align*}
		\E [R_K]
		&= \E\left[\sumk\sum_{j=1}^Np_k(j)\inner{q^j_k}{c_k} - \sumk\inner{q_{\optpi}}{c_k}\right] \\
		&= \E\left[\sumk\sum_{j=1}^Np_k(j)\inner{q^j_k}{c_k} - \inner{q_k^{j^\star}}{c_k}\right] + \E\left[\sumk\inner{q_k^{j^\star}-q_{\optpi}}{c_k}\right]\\
		&= \E\left[\sumk\inner{p_k-e_{j^\star}}{\ell_k}\right]  +  \E\left[\sumk\inner{q_k^{j^\star}-q_{\optpi}}{c_k}\right],
	\end{align*}
	where $e_{j^\star} \in \Omega$ is the basis vector with the $j^\star$-th coordinate being $1$.
	By the regret guarantee of \pref{alg:worker} (\pref{lem:worker}), the second term above is bounded by $\tilo{\sqrt{Db(j^\star) K}} = \tilo{\sqrt{D\T K}}$, which also means
     \begin{align*}
	\E\left[\sumk \ell_k(j^\star)\right] &\leq \E\left[\sumk \inner{q_{\optpi}}{c_k}\right] +  \tilo{\sqrt{D\T K}} \\
	&\leq \E\left[\sumk \inner{q_{\pi^f}}{c_k}\right] +  \tilo{\sqrt{D\T K}} \\
	&\leq DK + \tilo{\sqrt{D\T K}}.
	\end{align*}
	Using \pref{lem:multi-scale-expert}, the first term is bounded as 
	 \begin{align*}
	\E\left[\sumk\inner{p_k-e_{j^\star}}{\ell_k}\right] &=
	\tilO{\frac{1}{\eta_{j^\star}} + \eta_{j^\star}b(j^\star)\E\left[\sumk \ell_k(j^\star)\right]} \\
	&= \tilO{\frac{1}{\eta_{j^\star}} + \eta_{j^\star}D\T K + \eta_{j^\star}\T\sqrt{D\T K}} \\
	&= \tilO{\frac{1}{\eta_{j^\star}} + \eta_{j^\star}D\T K}.
	\end{align*}
     Finally plugging in the definition of $\eta_{j^\star}$ finishes the proof.
\end{proof}	
	
The lemma below is an improved guarantee for the multi-scale expert problem, which might be of independent interest.
\begin{lemma}\label{lem:multi-scale-expert}
For any $j^\star \in [N]$, \pref{alg:meta} ensures \[
\sumk\inner{p_k-e_{j^\star}}{\ell_k} = \frac{2+\ln\left(N\sqrt{\frac{b(j^\star)}{b(1)}}\right)}{\eta_{j^\star}} + 4\eta_{j^\star}b(j^\star)\sumk \ell_k(j^\star).
\]
\end{lemma}
\begin{proof}
Similar to \pref{eq:standard_OMD}, by standard OMD analysis (see also \citep[Lemma~6]{bubeck2017online}) we have:
	\begin{align}
	\sumk\inner{p_k-e_{j^\star}}{\ell_k+a_k} \leq D_{\psi}(e_{j^\star}, p_1) + \sumk\inner{p_k-p'_{k+1}}{\ell_k+a_k} \label{eq:standard_OMD2}
	\end{align}
	where $p'_{k+1}(j) = p_k(j)e^{-\eta_j (\ell_k(j)+a_k(j))}$.
	Using the inequality $1 - e^{-x} \leq x$ and the fact $a_k(j) \leq 4\eta_j b(j)\ell_k(j) \leq \ell_k(j)$ (since $\ell_k(j) \leq \sumsa q_k^j(s,a) \leq b(j)$ and $\eta_j\leq \frac{1}{4b(j)}$), we obtain
   	\begin{align*}  
   	\inner{p_k-p'_{k+1}}{\ell_k+a_k}
   	&\leq \sum_{j=1}^N\eta_jp_k(j)\left(\ell_k(j) + a_k(j)\right)^2 
   	\leq 4\sum_{j=1}^N\eta_jp_k(j) \ell_k^2(j) = \inner{p_k}{a_k}.
	\end{align*}
	Plugging this back into \pref{eq:standard_OMD2} and rearranging leads to
	\begin{align*}  
	\sumk\inner{p_k-e_{j^\star}}{\ell_k} &\leq D_{\psi}(e_{j^\star}, p_1) + \sumk a_k(j^\star) 
	= D_{\psi}(e_{j^\star}, p_1) + 4 \eta_{j^\star} \sumk \ell_k^2(j^\star) \\
	&\leq D_{\psi}(e_{j^\star}, p_1) + 4 \eta_{j^\star}b(j^\star) \sumk  \ell_k(j^\star).
	\end{align*}
	It remains to bound $D_{\psi}(e_{j^\star}, p_1)$, which by definition is 
	\[
	\sum_{j=1}^N\frac{1}{\eta_j}\left(e_{j^\star}(j)\ln\frac{e_{j^\star}(j)}{p_1(j)} - e_{j^\star}(j) + p_1(j)\right)
	\leq \frac{1}{\eta_{j^\star}}\ln\frac{1}{p_1(j^\star)} + \sum_{j=1}^N\frac{p_1(j)}{\eta_j}.
	\]
	Using the definition of $p_1$, when $j^\star \neq 1$ we have
	\[
	\frac{1}{\eta_{j^\star}}\ln\frac{1}{p_1(j^\star)} = \frac{1}{\eta_{j^\star}}\ln\left(\frac{N\eta_{1}}{\eta_{j^\star}}\right) = \frac{\ln\left(N\sqrt{\frac{b(j^\star)}{b(1)}}\right)}{\eta_{j^\star}};
	\]
	when $j^\star=1$, the same holds as an upper bound since $p_1(1) \geq 1/N$.
	Finally, the second term can be bounded as
	\[
	\sum_{j=1}^N\frac{p_1(j)}{\eta_j} = \frac{p_1(1)}{\eta_1} + \sum_{j\neq1}\frac{1}{N\eta_1} \leq \frac{2}{\eta_1} \leq \frac{2}{\eta_{j^\star}},
	\]
	which finishes the proof.
\end{proof}

\subsection{\pfref{thm:full-known-lb}}
\begin{proof}
	By Yao's minimax principle, in order to obtain a regret lower bound, it suffices to show that there exists a distribution of SSP instances that forces any deterministic learner to suffer a regret bound of $\lowO{\sqrt{D\T K}}$ in expectation.
	Below we describe such a distribution (the MDP is fixed but the costs are stochastic).
	\begin{itemize}
	\item The state space is $\calS = \{s_0, s_1,\ldots,s_N,f\}$ for any $N \geq 2$.
	\item At state $s_0$, there are $N$ available actions $a_1, \ldots, a_N$; at each state of $s_1, \ldots, s_N$, there are two available actions $a_g$ and $a_f$; and at state $f$, there is only one action $a_g$.
	\item At state $s_0$, taking action $a_j$ transits to state $s_j$ deterministically for all $j\in [N]$.
	At any state $s_j$ ($j\in [N]$), taking action $a_f$ transits to state $f$ deterministically, while taking action $a_g$ transits to the goal state $g$ with probability $1/\T$ and stays at the same state with probability $1-1/\T$.
	Finally, at state $f$, taking action $a_g$ transits to the goal state $g$ with probability $1/D$ and stays with probability $1-1/D$.
	\item The cost at state $s_0$ is always zero, that is, $c_k(s_0, a) = 0$ for all $k$ and $a$; the cost of action $a_f$ is also always zero, that is, $c_k(s, a_f) = 0$ for all $k$ and $s \in \{s_1, \ldots, s_N\}$;
	the cost at state $f$ is always one, that is, $c_k(f,a_g) = 1$ for all $k$;
	finally, the cost of taking action $a_g$ at state $s \in \{s_1, \ldots, s_N\}$ is generated stochastically as follows: first, a good state $j^\star \in [N]$ is sampled uniformly at random ahead of time and then fixed throughout the $K$ episodes;
	then, in each episode $k$, $c_k(s,a_g)$ is an independent sample of $\bernoulli(\frac{D}{2\T})$ if $s = s_{j^\star}$, and an independent sample of $\bernoulli(\frac{D}{2\T}+\epsilon)$ if $s \neq s_{j^\star}$, for some $\epsilon \leq \frac{D}{2\T}$ to be specified later.
	\end{itemize}
	
It is clear that in all these SSP instances, the diameter is $D+2$ (since one can reach the goal state via the fast state $f$ within at most $D+2$ steps in expectation), and the hitting time of the optimal policy is indeed $\T+1$ (in fact, the hitting time of any stationary deterministic policy is either $\T+1$ or $D+2 \leq \T+1$).
It remains to argue $\E[R_K] = \lowO{\sqrt{D\T K}}$ for any deterministic learner, where the expectation is over the randomness of the costs.
To do so, let $\E_j$ denote the conditional expectation given that the good state $j^\star$ is $j$.
Then we have
	\begin{align*}
		\E[R_K] &= \frac{1}{N}\sum_{j=1}^N\left(\mathbb{E}_j\left[\sumk\sum_{i=1}^{I_k}c_k(s_k^i, a_k^i) - \min_{\pi\in \PSD} \sumk J_k^\pi(s_0) \right]\right)\\
		&\geq \frac{1}{N}\sum_{j=1}^N\left(\E_j\left[\sumk\sum_{i=1}^{I_k}c_k(s_k^i, a_k^i) -   \sumk J_k^{\pi_j}(s_0) \right]\right),
	\end{align*}
where $\pi_j$ is the policy that picks action $a_j$ at state $s_0$ and $a_g$ at state $s_j$ (other states are irrelevant).
Note that it takes $\T$ steps in expectation for $\pi_j$ to reach $g$ from $s_j$ and each step incur expected cost $\frac{D}{2\T}$, which means $\E_j[J_k^{\pi_j}(s_0)] = \frac{D}{2\T} \times \T = \frac{D}{2}$.
	On the other hand, the learner is always better off not visiting $f$ at all, since starting from state $f$, the expected cost before reaching $g$ is $D$, while the expected cost of reaching the goal state via any other states is at most $\left(\frac{D}{2\T}+\epsilon\right) \times \T \leq D$.
	Therefore, depending on whether the learner selects the good action $a_{j^\star}$ or not at the first step, we further lower bound the expected regret as
	\begin{align*}	
		\E[R_K] &\geq \frac{1}{N}\sum_{j=1}^N\sumk\E_j\left[\frac{D}{2} + \T\epsilon\Ind\{a_k^1\neq a_j\} - \frac{D}{2}\right] 
		= \T K\epsilon - \frac{\T\epsilon}{N}\sum_{j=1}^N \E_j[K_j],
	\end{align*}
	where $K_j=\sumk \Ind\{a_k^1=a_j\}$.
	
	It thus suffices to upper bound $\sum_{j=1}^N \E_j[K_j]$. To do so, consider a reference environment without a good state, that is, $c_k(s, a_g)$ is an independent sample of $\bernoulli(\frac{D}{2\T}+\epsilon)$ for all $k$ and all $s\in\{s_1,\ldots,s_N\}$.
	Denote by $\E_0$ the expectation with respect to this reference environment, and by $P_0$ the distribution of the learner's observation in this environment ($P_j$ is defined similarly).
	Then with the fact $K_j \leq K$ and Pinsker's inequality, we have
	\[ \E_j[K_j] - \E_0[K_j] 
	\leq K\lVert P_j-P_0 \rVert_1\leq K\sqrt{2\KL(P_0, P_j)}.
	\]
	By the divergence decomposition lemma (see e.g.~\citep[Lemma~15.1]{lattimore2020bandit}) and the nature of the full-information setting, we further have
\begin{align*}
\KL(P_0, P_j) &= \sum_{j'=1}^N \E_0[K_{j'}] \times \KL\left(\bernoulli\left(\frac{D}{2\T}+\epsilon\right), \bernoulli\left(\frac{D}{2\T}\right)\right) \\
&= K \times \KL\left(\bernoulli\left(\frac{D}{2\T}+\epsilon\right), \bernoulli\left(\frac{D}{2\T}\right)\right) \\
&\leq \frac{K\epsilon^2}{\alpha(1-\alpha)},
\end{align*}
where the last step is by \citep[Lemma~6]{gerchinovitz2016refined} with $\alpha = \frac{D}{2\T}$.
	Therefore, we have
	\begin{align*}
		\sum_{j=1}^N \E_j[K_j] &\leq \sum_{j=1}^N\E_0[K_j] + NK\sqrt{\frac{2K\epsilon^2}{\alpha(1-\alpha)}}
		= K + NK\sqrt{\frac{2K\epsilon^2}{\alpha(1-\alpha)}}.
	\end{align*}
This is enough to show the claimed lower bound:
	\begin{align*}
		\E [R_K] &\geq \T\epsilon K - \frac{\T\epsilon}{N}\sum_{j=1}^N \E_j[K_j] \\
		&\geq \T\epsilon K - \frac{\T\epsilon}{N}\left[K+NK\sqrt{\frac{2K\epsilon^2}{\alpha(1-\alpha)}}\right]\\
		&= \T\epsilon K\left[1 - \frac{1}{N} - \epsilon\sqrt{\frac{2K}{\alpha(1-\alpha)}}\right]\\
		&\geq \T\epsilon K\left[\frac{1}{2} - \epsilon\sqrt{\frac{2K}{\alpha(1-\alpha)}}\right]\\
		&= \frac{\T K}{16}\sqrt{\frac{\alpha(1-\alpha)}{2K}} =\lowO{\sqrt{D\T K}},
	\end{align*}
	where in the last line we choose $\epsilon=\frac{1}{4}\sqrt{\frac{\alpha(1-\alpha)}{2K}} \leq \frac{1}{8}\sqrt{\frac{D}{\T K}} \leq \frac{D}{2\T}$ to maximize the lower bound.
\end{proof}

\subsection{\pfref{lem:deviation}}
\begin{proof}
	With the inequality $(\sum_{i=1}^I a_i)^2\leq 2\sum_{i}a_i(\sum_{i'= i}^I a_{i'})$, we proceed as
	\begin{align*}
		&\E\left[\left( \sumsa N_{k}(s, a)c_k(s, a) \right)^2\right] \\
		&= \E\left[\left( \sum_{i=1}^{I_k}\sumsa \Ind\{s^i_k=s, a^i_k=a\}c_k(s, a) \right)^2\right] \\
		&\leq 2\E\left[\sum_{i=1}^{I_k}\sumsa \Ind\{s^i_k=s, a^i_k=a\}\left( \sum_{i'= i}^{I_k}\sumsa[s', a'] \Ind\{s^{i'}_k=s', a^{i'}_k=a'\}c_k(s', a') \right)\right] \\
		&= 2\E\left[\sum_{i=1}^{I_k}\sums \Ind\{s^i_k=s\}\E\left[\left. \sum_{i'= i}^{I_k}\sumsa[s', a'] \Ind\{s^{i'}_k=s', a^{i'}_k=a'\}c_k(s', a') \right| s^i_k=s\right]\right]\\
		&= 2\E\left[\sum_{i=1}^{I_k}\sums \Ind\{s^i_k=s\}J^{\pi}_k(s)\right]\\
		&= 2\sums q_{\pi}(s)J^{\pi}_k(s) = 2\inner{q_{\pi}}{J^{\pi}_k},
	\end{align*}
	completing the proof.
\end{proof}

\subsection{\pfref{lem:deviation_loop_free}}
\begin{proof}
	Applying \pref{lem:deviation} (to the loop-free instance), we have
	\begin{align*}
		\E\left[\inner{\tilN_{k}}{c_k}^2\right]
		&\leq 2\sum_{\tils\in\tilS}q_{\tilpi}(\tils)J^{\tilpi}_k(\tils)
		= 2\sum_{h=1}^H\sum_{s\in\calS\cup\{s_f\}}q_{\tilpi}(s, h)J^{\tilpi}_k(s, h).
	\end{align*}
	Denote $q_{\tilpi, (s, h)}$ as the occupancy measure of policy $\tilpi$ with initial state $(s, h)$, so that 
	\[
	J^{\tilpi}_k(s, h) = \sum_{(s', a')\in\tilSA}\sum_{h'\geq h}q_{\tilpi,(s, h)}(s', a', h')c_k(s', a', h').
	\]
	Then we continue with the following equalities:
	\begin{align}
		&\sum_{h=1}^H\sum_{s\in\calS\cup\{s_f\}}q_{\tilpi}(s, h)J^{\tilpi}_k(s, h) \notag \\
		&= \sum_{h=1}^{H}\sum_{s\in\calS\cup\{s_f\}} q_{\tilpi}(s, h)\sum_{(s', a')\in\tilSA}\sum_{h'\geq h}q_{\tilpi,(s, h)}(s', a', h')c_k(s', a', h') \notag\\
		&= \sum_{h=1}^{H}\sum_{(s', a')\in\tilSA}\sum_{h'\geq h}\left(\sum_{s\in\calS\cup\{s_f\}} q_{\tilpi}(s, h)q_{\tilpi,(s, h)}(s', a', h')\right) c_k(s', a', h') \notag\\
		&= \sum_{h=1}^{H}\sum_{(s', a')\in\tilSA}\sum_{h'\geq h}q_{\tilpi}(s', a', h')c_k(s', a', h') \notag\\
		&= \sum_{h=1}^{H}\sum_{(s, a)\in\tilSA} h\cdot q_{\tilpi}(s, a, h)c_k(s, a, h)
		= \inner{q_{\tilpi}}{\h{c_k}}. \label{eq:loop_free_variance}
	\end{align}
	where in the third line we use the equality $\sum_{s\in\calS\cup\{s_f\}} q_{\tilpi}(s, h)q_{\tilpi,(s, h)}(s', a', h') = q_{\tilpi}(s', a', h')$ by definition (since both sides are the probability of visiting $(s', a', h')$). This completes the proof.
\end{proof}

\subsection{\pfref{lem:loop-free}}
\begin{proof}
We first prove the second statement $\E[R_K] \leq \E[\reg] + \tilO{1}$.
Since the fast policy reaches the goal state within $D$ steps in expectation starting from any state,
by the definition of $\sigma(\tilpi)$ and $\tilM$, we have $J^{\sigma(\tilpi)}_k(s_0)\leq J^{\tilpi}_k(\tils_0)$ for any $\tilpi$, that is, the expected cost of executing $\sigma(\tilpi)$ in $M$ is not larger than that of executing $\tilpi$ in $\tilM$.
On the other hand, since the probability of not reaching the goal state within $H_1$ steps when executing $\optpi$ is at most:
$
 2e^{-\frac{H_1}{4\Tmax}} \leq \frac{2}{K^2}
$
by \pref{lem:hitting} and the choice of $H_1$,
the expected cost of $\optpi$ in $M$ and the expected cost of $\tiloptpi$ in $\tilM$ is very similar:
\begin{equation}\label{eq:optpi_similar}
J^{\tiloptpi}_k(\tils_0)  \leq J^{\optpi}_k(s_0) + \frac{2H_2}{K^2} =  
J^{\optpi}_k(s_0) + \tilO{\frac{1}{K}}.
\end{equation}
This proves the second statement:
	\begin{align*}
		\E[R_K] &= \E\sbr{\sumk J^{\sigma(\tilpi_k)}_k(s_0) - J^{\optpi}_k(s_0)}\\
		&\leq \E\sbr{\sumk J^{\tilpi_k}_k(\tils_0) - J^{\tiloptpi}_k(\tils_0)} + \tilO{1}
		= \E[\reg] + \tilO{1}.
	\end{align*}
	To prove the first statement, we apply \pref{lem:hitting} again to show that for each episode $k$, the probability of the learner not reaching $g$ within $H$ steps is at most $ 2e^{-\frac{H_2}{4D}} = \frac{\delta}{2K}$.
	With a union bound, this means, with probability at least $1-\frac{\delta}{2}$, the learner reaches the goal within $H$ steps for all episodes and thus her actual loss in $M$ is not larger than that in $\tilM$: $\sumk \inner{N_k}{c_k} \leq \sumk \inner{\tilN_k}{c_k}$.  Together with \pref{eq:optpi_similar}, this shows
	\[
	R_K \leq \sumk\inner{\tilN_{k}-\tiloptq}{c_k} + \tilO{1}.
	\]
	It thus remains to bound the deviation $\sumk \inner{\tilN_k - q_k}{c_k}$, which is the sum of a martingale difference sequence.
We apply Freedman's inequality \pref{lem:freedman} directly:
the variable $\inner{\tilN_k}{c_k}$ is bounded by $H$ always, and its conditional variance is bounded by $2\inner{q_k}{\h{c_k}}$ as shown in \pref{lem:deviation_loop_free}, which means for any $\lambda \in (0, 2/H]$,
\[
\sumk \inner{\tilN_k - q_k}{c_k}\leq \lambda\sumk\inner{q_k}{\h{c_k}} + \frac{2\ln\left(\nicefrac{2}{\delta}\right)}{\lambda}
\]
holds with probability at least $1-\frac{\delta}{2}$.
Applying another union bound finishes the proof.
\end{proof}

\begin{lemma}{\citep[Lemma E.1]{rosenberg2020adversarial}}\label{lem:hitting}
	Let $\pi$ be a policy with expected hitting time at most $\tau$ starting from any state.
	Then, the probability that $\pi$ takes more than $m$ steps to reach the goal state is at most $2e^{-\frac{m}{4\tau}}$.
\end{lemma}

\subsection{\pfref{thm:full-known-hp}}
For completeness, we first spell out the definition of $\tilDelta(T)$, which is the exact counterpart of $\Delta(T)$ defined in \pref{eq:Delta_T} for $\tilM$ (the first equality below), but can be simplified using the special structure of $\tilP$ (the second equality below).
\begin{align}
\tilDelta(T) &= \Bigg\{ q \in [0,1]^{\tilSA\times [H]}: \sumh\sumtilsaf q(s, a, h)\leq T, \notag\\
&\hspace{-30pt}\sum_{a\in\tilA_{(s,h)}} q(s, a, h) - \sum_{h'=1}^H\sumtilsaf[s', a'] \tilP((s, h)| (s', h'), a')q(s', a', h') = \Ind\{(s,h)=\tils_0\}, \;\forall (s,h)\in\tilS  \Bigg\} \notag \\
&= \Bigg\{ q \in [0,1]^{\tilSA\times [H]}:  \sumh\sumtilsaf q(s, a, h)\leq T, \quad \sum_{a\in\calA_{s_0}} q(s_0,a,1)=1, \; \notag \\
&\qquad\qquad q(s,a,1)=0, \; \forall s\neq s_0 \text{ and } a\in\calA_s, \qquad q(s, a, h) = 0, \; \forall (s,a)\in\SA \text{ and } h > H_1, \notag \\
&\qquad\qquad q(s_f, a_f, h) =  
\Ind\{h > H_1\} \sumsaf[s',a'] q(s', a', H_1), \notag\\
& \quad\quad \sum_{a\in\calA_s} q(s,a,h) = \sumsaf[s',a'] P(s |s',a') q(s',a', h-1), \; \forall s \in \calS \text{ and } 1<h \leq H_1.
\Bigg\} \label{eq:tilDelta_T}
\end{align}

Note that $q_\tiloptpi$ belongs to $\tilDelta(\T+1)$ as shown in the following lemma.
\begin{lemma}\label{lem:tiloptpi_hitting_time}
The policy $\tiloptpi$ satisfies $T^{\tiloptpi}(\tils_0) = \sumh\sumtilsaf q_\tiloptpi(s,a,h) \leq \T+1$ and thus $q_\tiloptpi \in \tilDelta(\T+1)$.
\end{lemma}
\begin{proof}
This is a direct application of the fact $T^{\optpi}(s_0) = \T$ and \pref{lem:hitting}: the probability of not reaching the goal state within $H_1$ steps when executing $\optpi$ is at most:
$
 2e^{-\frac{H_1}{4\Tmax}} \leq \frac{2}{K^2}.
$
Therefore, $T^{\tiloptpi}(\tils_0) \leq T^{\optpi}(s_0) + \frac{2H_2 }{K^2} \leq \T+1$, finishing the proof.
\end{proof}

We also need the following lemma.
\begin{lemma}\label{lem:tiloptpi_variance}
The policy $\tiloptpi$ satisfies $\sumk\inner{q_{\tiloptpi}}{\h{c_k}} = \bigO{D\T K}$.
\end{lemma}
\begin{proof}
We proceed as follows:
\begin{align*}
	\sumk\inner{q_{\tiloptpi}}{\h{c_k}}
	&= \sumk\sum_{h=1}^H\sum_{s\in\calS\cup\{s_f\}}q_{\tiloptpi}(s, h)J^{\tiloptpi}_k(s, h) \\
	&= \sum_{h=1}^H\sum_{s\in\calS\cup\{s_f\}}q_{\tiloptpi}(s, h) \sumk J^{\tiloptpi}_k(s, h) \\
&\leq \sum_{h=1}^H\sum_{s\in\calS\cup\{s_f\}}q_{\tiloptpi}(s, h) \sumk J^{\tiloptpi}_k(s, 1) \\
&\leq \sum_{h=1}^H\sum_{s\in\calS\cup\{s_f\}}q_{\tiloptpi}(s, h)\left(\tilO{1} + \sumk J^{\optpi}_k(s)\right) \\
&\leq \sum_{h=1}^H\sum_{s\in\calS\cup\{s_f\}}q_{\tiloptpi}(s, h)\left(\tilO{1} + \sumk J^{\pi^f}_k(s)\right) \\
&\leq \left(\tilO{1} + DK\right) \sum_{h=1}^H\sum_{s\in\calS\cup\{s_f\}}q_{\tiloptpi}(s, h) \\
&\leq \tilO{D\T K},
	\end{align*}
where the first line is by \pref{eq:loop_free_variance}, the fourth line is by the same reasoning of \pref{eq:optpi_similar}, and the last line is by \pref{lem:tiloptpi_hitting_time}.
\end{proof}

We are now ready to prove \pref{thm:full-known-hp}.
\begin{proof}
Define $\qfeat^{\star} = \tiloptq + \lambda \h{\tiloptq}$.
which belongs to the set $\Omega$ by \pref{lem:tiloptpi_hitting_time} and the condition $T\geq \T+1$.
By the exact same reasoning of \pref{eq:OMD_small_loss} in the proof of \pref{lem:worker}, OMD ensures
\[
\sumk\inner{\qfeat_k - \qfeat^\star}{c_k} 
		\leq 2D_{\psi}(\qfeat^\star, \qfeat_1)+2\eta\sumk\inner{\qfeat^\star}{c_k}.
\]
The last two terms can also be bounded in a similar way as in the proof of \pref{lem:worker}: for the first term,
	since $\qfeat_1$ minimizes $\psi$ over $\Omega$, we have $\inner{\nabla\psi(\qfeat_1)}{\qfeat^\star-\qfeat_1}\geq 0$, and thus with the fact $\sumh\sumsa \qfeat(s,a,h) \leq T + \lambda H T \leq 2T$ for any $\qfeat \in \Omega$ we obtain
	\begin{align*}
		&D_{\psi}(\qfeat^\star, \qfeat_1) \leq \psi(\qfeat^\star) - \psi(\qfeat_1)  \\
		&= \frac{1}{\eta}\sumh\sumsa \qfeat^\star(s,a,h) \ln \qfeat^\star(s,a,h) - \frac{1}{\eta}\sumh\sumsa \qfeat_1(s,a,h) \ln \qfeat_1(s,a,h) \\
		&\leq \frac{2T \ln(2T)}{\eta}  - \frac{2T}{\eta}\sumh\sumsa \frac{\qfeat_1(s,a,h)}{2T} \ln \frac{\qfeat_1(s,a,h)}{2T} \\
		&\leq \frac{2T\ln(2T)}{\eta} + \frac{2T\ln(|\tilSA|H)}{\eta} \\
		&= \bigO{\frac{T\ln(SAHT)}{\eta}};
	\end{align*}
	for the second term, we have
	\begin{align*}
	\sumk\inner{\qfeat_{\tiloptpi}}{c_k} &\leq 2\sumk\inner{q_{\tiloptpi}}{c_k} 
	\leq 2\sumk J^{\optpi}_k(s_0) + \tilO{1}  \\
	&\leq 2\sumk J^{\pi^f}_k(s_0) + \tilO{1} \leq 2DK + \tilO{1},
	\end{align*}
	where the second inequality is by \pref{eq:optpi_similar}.
	Combining the above and plugging the choice of $\eta$, we arrive at
	\begin{align*}
		\sumk\inner{\qfeat_k - \qfeat^\star}{c_k} 
		&\leq \bigO{\frac{T\ln(SAHT)}{\eta}} + 2\eta DK + \tilO{1}
		= \tilO{\sqrt{DT K}}.
	\end{align*}
	Finally, we apply \pref{lem:loop-free}: with probability at least $1-\delta$, 
	\begin{align*}
	R_K &\leq \sumk\inner{q_k - q_{\tiloptpi}}{c_k} + \lambda\sumk\inner{q_k}{\h{c_k}} + \frac{2\ln\left(\nicefrac{2}{\delta}\right)}{\lambda} + \tilO{1} \\
	&= \sumk\inner{\qfeat_k - \qfeat^\star}{c_k}  + \lambda\sumk\inner{q_{\tiloptpi}}{\h{c_k}} + \frac{2\ln\left(\nicefrac{2}{\delta}\right)}{\lambda} + \tilO{1} \\
	&= \tilO{\sqrt{DT K}} + \lambda\sumk\inner{q_{\tiloptpi}}{\h{c_k}} + \frac{2\ln\left(\nicefrac{2}{\delta}\right)}{\lambda} \\
	&=  \tilO{\sqrt{DT K}} + \tilO{\lambda DTK} + \frac{2\ln\left(\nicefrac{2}{\delta}\right)}{\lambda} \tag{\pref{lem:tiloptpi_variance}}\\
	&= \tilO{\sqrt{DT K \ln\left(\nicefrac{1}{\delta}\right)}} \tag{by the choice of $\lambda$},
	\end{align*}
	which finishes the proof.
\end{proof}

\section{Omitted details for \pref{sec:bandit-known}}\label{app:bandit-known}
% !TEX root = main.tex

In this section, we provide all omitted algorithms and proofs for \pref{sec:bandit-known}.

\subsection{Optimal Expected Regret}

%Our algorithm with optimal expected regret for oblivious adversaries is presented in \pref{alg:bandit-known-pseudo}. \\

\begin{proof}[of \pref{thm:bandit-known-pseudo}]
Using the second statement of \pref{lem:loop-free}, we have
\[
\E\left[R_K\right] = \E\sbr{\sumk\inner{q_k-\tiloptq}{c_k}} + \tilo{1}.
\]
As in all analysis for OMD with log-barrier regularizer, we consider a slightly perturbed benchmark $q^\star = (1-\frac{1}{TK})\tiloptq + \frac{1}{TK}q_1$ which is in $\tilDelta(T)$ by the convexity of $\tilDelta(T)$, the condition $T\geq \T+1$, and \pref{lem:tiloptpi_hitting_time}.
We then have
\begin{align*}
\E\left[R_K\right] &\leq \E\sbr{\sumk\inner{q_k-q^\star}{c_k}} + \frac{1}{TK-1}\E\sbr{\sumk\inner{q_1}{c_k}} + \tilo{1} \\
&=\E\sbr{\sumk\inner{q_k-q^\star}{c_k}} + \tilO{1}.
\end{align*}
It remains to bound $\E\sbr{\sumk\inner{q_k-q^\star}{c_k}}$.
Let $\qfeat^\star=q^\star+\lambda\h{q^\star} \in \Omega$.
By the non-negativity and the unbiasedness of the cost estimator, the obliviousness of the adversary, and the same argument of~\citep[Lemma~12]{agarwal2017corralling}, OMD with log-barrier regularizer ensures
\begin{align*}
	\E\sbr{\sumk\inner{\qfeat_k - \qfeat^\star}{c_k}} &= \E\sbr{\sumk\inner{\qfeat_k - \qfeat^\star}{\hatc_k}} 
	\leq D_{\psi}(\qfeat^\star, \qfeat_1) + \eta\E\sbr{\sumk\sumsa \qfeat_k^2(s, a)\hatc_k^2(s, a) }.
\end{align*}	
For the first term, as $\qfeat_1$ minimizes $\regz$, we have $\inner{\nabla\regz(\qfeat_1)}{\qfeat^\star-\qfeat_1}\geq 0$
and thus
\begin{align*}
D_{\psi}(\qfeat^\star, \qfeat_1)
&\leq \frac{1}{\eta}\sumsa\ln \frac{\qfeat_1(s,a)}{\qfeat^\star(s,a)} 
= \frac{SA}{\eta} \ln(HT) = \tilO{\frac{SA}{\eta}}.
\end{align*}
For the second term, we note that
\begin{align*}
\E\sbr{\sumsa \qfeat_k^2(s, a)\hatc_k^2(s, a) }
&\leq 4\E\sbr{\sumsa q_k^2(s, a)\hatc_k^2(s, a)} 
= 4\E\sbr{\sumsa \tilN_k^2(s, a) c^2_k(s, a)} \\
&\leq 4\E\sbr{\inner{\tilN_k}{c_k}^2}
\leq 8\E\sbr{\inner{q_k}{\h{c_k}}},
\end{align*}
where the last step is by \pref{lem:deviation_loop_free}.
Combining everything, we have shown
\[
\E\sbr{\sumk\inner{\qfeat_k - \qfeat^\star}{c_k}}
= \tilO{\frac{SA}{\eta}} + 8\eta\E\sbr{\sumk \inner{q_k}{\h{c_k}}},
\]
and thus
\begin{align*}
\E\sbr{\sumk\inner{q_k-q^\star}{c_k}}
&= \E\sbr{\sumk\inner{\qfeat_k-\qfeat^\star}{c_k}} + \lambda \E\sbr{\sumk \inner{q^\star}{\h{c_k}}} - \lambda \E\sbr{\sumk \inner{q_k}{\h{c_k}}} \\
&= \tilO{\frac{SA}{\eta}} + 8\eta\E\sbr{\sumk \inner{q^\star}{\h{c_k}}} \tag{$\lambda = 8\eta$} \\
&= \tilO{\frac{SA}{\eta} + \eta DTK} \tag{\pref{lem:tiloptpi_variance}}.
\end{align*}
Plugging the choice of $\eta$ finishes the proof.
\end{proof}

 \subsection{\pfref{thm:bandit-known-lb}}
 \begin{proof}
	By Yao's minimax principle, in order to obtain a regret lower bound, it suffices to show that there exists a distribution of SSP instances that forces any deterministic learner to suffer a regret bound of $\lowO{\sqrt{D\T SAK}}$ in expectation.
	We use the exact same construction as in \pref{thm:full-known-lb} with $N=S-2$ (note that the average number of actions $A$ is $\bigo{1}$).
	The proof is the same up to the point where we show 
	\begin{align*}	
		\E[R_K] \geq \T K\epsilon - \frac{\T\epsilon}{N}\sum_{j=1}^N \E_j[K_j],
	\end{align*}
	with $K_j=\sumk \Ind\{a_k^1=a_j\}$, and
	\[ 
	\E_j[K_j] - \E_0[K_j] \leq K\lVert P_j-P_0 \rVert_1\leq K\sqrt{2\KL(P_0, P_j)}.
	\]
	What is different is the usage of the divergence decomposition lemma (see e.g.~\citep[Lemma~15.1]{lattimore2020bandit}) due to the different observation model:
	\begin{align*}
		\KL(P_0, P_j) &= \sum_{j'=1}^N \E_0[K_{j'}] \times \KL\left(\bernoulli\left(\frac{D}{2\T}+\epsilon\right), \bernoulli\left(\frac{D}{2\T}+\epsilon\Ind\{j'\neq j\}\right)\right) \\
		&= \E_0[K_j] \times \KL\left(\bernoulli\left(\frac{D}{2\T}+\epsilon\right), \bernoulli\left(\frac{D}{2\T}\right)\right) \\
		&\leq \frac{\E_0[K_j]\epsilon^2}{\alpha(1-\alpha)},
	\end{align*}
	where the last step is again by \citep[Lemma~6]{gerchinovitz2016refined} with $\alpha = \frac{D}{2\T}$.
	Therefore, we can upper bound $\sum_{j=1}^N\E_j[K_j]$ as:
	\begin{align*}
		\sum_{j=1}^N\E_j[K_j] &\leq \sum_{j=1}^N\E_0[K_j] + K\sqrt{\frac{2\epsilon^2}{\alpha(1-\alpha)}}\sum_{j=1}^N\sqrt{\E_0[K_j]}\\
		&\leq\sum_{j=1}^N\E_0[K_j] + K\sqrt{\frac{2N\epsilon^2}{\alpha(1-\alpha)}\sum_{j=1}^N\E_0[K_j]} \tag{Cauchy-Schwarz inequality}\\
		&= K + K\sqrt{\frac{2NK\epsilon^2}{\alpha(1-\alpha)}} \tag{$\sum_{j=1}^N\E_0[K_j]=K$}.
	\end{align*}
	This shows the following lower bound:
	\begin{align*}
		\E[R_K] &\geq \T\epsilon K - \frac{\T\epsilon}{N}\sum_{j=1}^N\E_j[K_j]\\
		&\geq \T\epsilon K - \frac{\T\epsilon}{N}\left[K+K\sqrt{\frac{2NK\epsilon^2}{\alpha(1-\alpha)}}\right]\\
		&= \T\epsilon K\left[1 - \frac{1}{N} - \epsilon\sqrt{\frac{2K}{N\alpha(1-\alpha)}}\right]\\
		&\geq \T\epsilon K\left[\frac{1}{2} - \epsilon\sqrt{\frac{2K}{N\alpha(1-\alpha)}}\right]\\
		&= \frac{\T K}{16}\sqrt{\frac{N\alpha(1-\alpha)}{2K}} = 
		\lowO{\sqrt{D\T  NK}} = \lowO{\sqrt{D\T SAK}},
	\end{align*}
	where in the last step we set $\epsilon=\frac{1}{4}\sqrt{\frac{N\alpha(1-\alpha)}{2K}} \leq \frac{1}{8}\sqrt{\frac{SD}{\T K}} \leq \frac{D}{2\T}$ to maximize the lower bound.
\end{proof}

\subsection{Optimal High-probability Regret}

We present our algorithm with optimal high-probability regret in \pref{alg:bandit-known-hp}.
The key difference compared to \pref{alg:bandit-known-pseudo} is the use of the extra bias term $\hatb_k$ in the OMD update and the time-varying individual learning $\eta_k(s,a)$ for each state-action pair together with an increasing learning rate schedule (see the last for loop).
Note that, similar to~\citep{lee2020bias}, the decision set $\Omega$ has the extra constraint $q(s,a) \geq \frac{1}{TK^4}$ compared to \pref{alg:full-known-hp} and \pref{alg:bandit-known-pseudo},
and it is always non-empty as long as $K$ is large enough and every state is reachable within $H$ steps starting from $s_0$ (states not satisfying this can simply be removed without affecting $\tilM$).

\begin{algorithm}[t]
\caption{Log-barrier Policy Search for SSP (High Probability)}
\label{alg:bandit-known-hp}
	\textbf{Input:} Upper bound on expected hitting time $T$, horizon parameter $H_1$, and confidence level $\delta$
		
     \textbf{Parameters:} $H_2=\ceil{4D\ln\frac{4K}{\delta}}$, $H=H_1+H_2$, $ C=\ceil{\log_2 (TK^4)}\ceil{\log_2 (T^2K^9)}, \beta=e^{\frac{1}{7\ln K}}, \eta=\sqrt{\frac{SA\ln\nicefrac{1}{\delta}}{D\T K}}, \gamma=100\eta\ln K\rbr{1+C\sqrt{8\ln\frac{CSA}{\delta}}}^2, \lambda=40\eta + 2\gamma$. 
	
	\textbf{Define:} regularizer $\regz_k(\qfeat) = \sum_{(s,a)\in\tilSA}\frac{1}{\eta_k(s,a)} \ln \frac{1}{\qfeat(s, a)}$ where $\qfeat(s, a) = \sumh\qfeat(s, a, h)$
	
	\textbf{Define:} decision set $\Omega = \{\qfeat=q + \lambda \h{q} : q\in \tilDelta(T), \;\; q(s,a) \geq \frac{1}{TK^4}, \;\forall (s,a)\in \tilSA\}$ \\
	
	\textbf{Initialization:} $\qfeat_1 = q_{1}+ \lambda\h{q_{1}} = \argmin_{\qfeat\in\Omega} \regz_1(\qfeat)$.
	
	\textbf{Initialization:} for all $(s, a)\in\tilSA, \eta_1(s, a)=\eta, \rho_1(s, a)=2T$. 
	
	\For{$k=1,\ldots, K$}{
		Execute $\sigma(\tilpi_k)$ where $\tilpi_k$ is such that $\tilpi_k(a| (s, h)) \propto q_k(s, a, h)$.
		
		Construct cost estimator $\hatc_k \in \fR_{\geq 0}^{\tilSA}$ such that $\hatc_k(s,a) = \frac{\tilN_{k}(s, a)c_k(s, a)}{q_{k}(s, a)}$
where $\tilN_{k}(s, a) = \sum_h \tilN_{k}(s, a, h)$ and $q_k(s,a) = \sum_h q_k(s,a,h)$
($\tilN_{k}$ is defined after \pref{lem:deviation_loop_free}).
         
         Construct bias term $\hatb_k\in \fR_{\geq 0}^{\tilSA}$ such that $\hatb_k(s, a) = \frac{\sum_h h  q_k(s, a, h)\hatc_k(s, a)}{q_k(s, a)}$.

		Update \[
			\qfeat_{k+1} = q_{k+1}+ \lambda\h{q_{k+1}} = \argmin_{\qfeat\in\Omega}\sumsa  \qfeat(s, a)\left(\hatc_k(s, a)-\gamma\hatb_k(s,a)\right) + D_{\psi_k}(\qfeat, \qfeat_k).
		\]
		
		\For{$\forall (s, a)\in\tilSA$}{
			\If{$\frac{1}{\qfeat_{k+1}(s, a)} > \rho_k(s, a)$}{
				$\rho_{k+1}(s, a)=\frac{2}{\qfeat_{k+1}(s, a)}, \eta_{k+1}(s, a)=\beta\eta_k(s, a)$.
			}
			\Else{
			$\rho_{k+1}(s, a)=\rho_{k}(s, a), \eta_{k+1}(s, a)=\eta_k(s, a)$.
			}
		}
	}
\end{algorithm}

Below we present the proof of \pref{thm:bandit-known-hp}. It decomposes the regret into several terms, each of which is bounded by a lemma included after the proof. \\

\begin{proof}[of \pref{thm:bandit-known-hp}]
We apply the first statement of \pref{lem:loop-free}: with probability $1-\delta$,
\[
	R_K \leq \sumk\inner{\tilN_{k}-\tiloptq}{c_k} + \tilO{1}.
\]
Similar to the proof of \pref{thm:bandit-known-pseudo}, we define
a slightly perturbed benchmark $q^\star = (1-\frac{1}{TK})\tiloptq + \frac{1}{TK}q_0 \in \tilDelta(T)$ for some $q_0 \in \tilDelta(T)$ with $q_0(s,a) \geq \frac{1}{K^3}$ for all $(s,a) \in \tilSA$ (which again exists as long as $K$ is large enough),
so that $R_K \leq \sumk\inner{\tilN_{k}-q^\star}{c_k} + \tilo{1}$ still holds.
Also define $\qfeat^\star=q^\star+\lambda\h{q^\star} \in \Omega$
and $b_k \in \fR^{\tilSA}$ such that $b_k(s,a) = \frac{\sum_h h  q_k(s, a, h)c_k(s, a)}{q_k(s, a)}$, which clearly satisfies $\E_k[\hatb_k] = b_k$.
We then decompose $\sumk\inner{\tilN_{k}-q^\star}{c_k}$ as
\begin{align*}
&\sumk\inner{\tilN_{k}-q^\star}{c_k} \\
&= \sumk\inner{q_{k}}{\hatc_k} - \sumk\inner{q^\star}{c_k}
\tag{$\inners{\tilN_{k}}{c_k} = \inner{q_k}{\hatc_k}$} \\
&= \sumk\inner{\qfeat_{k}-\qfeat^\star}{\hatc_k} + \sumk\inner{\qfeat^\star}{\hatc_k-c_k} + \lambda\sumk\inner{\h{q^\star}}{c_k} - \lambda\sumk\inner{\h{q_k}}{\hatc_k} \\
&= \sumk\inner{\qfeat_{k}-\qfeat^\star}{\hatc_k} + \sumk\inner{\qfeat^\star}{\hatc_k-c_k} + \tilO{\lambda DTK} - \lambda\sumk\inner{\h{q_k}}{\hatc_k} \tag{\pref{lem:tiloptpi_variance}} \\
&= \sumk\inner{\qfeat_{k}-\qfeat^\star}{\hatc_k} + \tilO{\lambda DTK} + \dev_1 + \dev_2 - \lambda\sumk\inner{\h{q_k}}{c_k}
\tag{define $\dev_1 = \sumk\inner{\qfeat^\star}{\hatc_k-c_k}$ and $\dev_2 = \lambda\sumk\inner{\h{q_k}}{c_k - \hatc_k}$} \\
&= \reg_\qfeat + \tilO{\lambda DTK} + \dev_1 + \dev_2 +  \gamma\sumk \inner{\qfeat_{k}-\qfeat^\star}{\hatb_k} - \lambda\sumk\inner{\h{q_k}}{c_k}
\tag{define $\reg_\qfeat = \sumk\inners{\qfeat_{k}-\qfeat^\star}{\hatc_k-\gamma\hatb_k}$} \\
&= \reg_\qfeat + \tilO{\lambda DTK} + \dev_1 + \dev_2 + \dev_3 + \dev_4 \\
&\hspace{5cm} + \gamma\sumk \inner{\qfeat_{k}-\qfeat^\star}{b_k} - \lambda\sumk\inner{\h{q_k}}{c_k}
\tag{define $\dev_3 = \gamma\sumk \inners{\qfeat_{k}}{\hatb_k - b_k}$ and  $\dev_4 = \gamma\sumk \inners{\qfeat^\star}{b_k - \hatb_k }$} \\
&\leq \reg_\qfeat + \tilO{\lambda DTK} + \dev_1 + \dev_2 + \dev_3 + \dev_4 \\
&\hspace{5cm}  +2\gamma \sumk \inner{q_k}{b_k} - \gamma \sumk\inner{\qfeat^\star}{b_k} - \lambda\sumk\inner{\h{q_k}}{c_k} \\
&= \reg_\qfeat + \tilO{\lambda DTK} + \dev_1 + \dev_2 + \dev_3 + \dev_4 \\ 
&\hspace{5cm} + (2\gamma-\lambda) \sumk \inner{q_k}{\h{c_k}} - \gamma \sumk\inner{\qfeat^\star}{b_k}.
\tag{$\inner{q_k}{b_k} = \inner{q_k}{\h{c_k}}$}
\end{align*}

The $\reg_\qfeat$ term can be upper bounded by the OMD analysis (see \pref{lem:OMD-increasing-eta}),
and the four deviation terms $\dev_1, \dev_2, \dev_3$, and $\dev_4$ are all sums of martingale difference sequences and can be bounded using Azuma's or Freedman's inequality (see \pref{lem:DEV_1_2} and \pref{lem:DEV_3_4}).
Combining everything, we obtain
\begin{align*}
	&R_K \leq \tilO{\frac{SA}{\eta}} - \frac{\inner{\phi^\star}{\rho_K}}{70\eta\ln K} + 40\eta\sumk\inner{q_k}{\h{c_k}} + \tilO{\lambda DTK}\\
	&+ \rbr{1+C\sqrt{8\ln\rbr{\frac{CSA}{\delta}}}}\rbr{\frac{\inner{\qfeat^\star}{\rho_K}}{\eta'}+\eta'\inner{\phi^\star}{\sumk b_k}} + \rbr{4CH\ln\rbr{\frac{CSA}{\delta}}}\inner{\phi^\star}{\rho_K}\\
	&+ (2\gamma-\lambda) \sumk \inner{q_k}{\h{c_k}} - \gamma \sumk\inner{\qfeat^\star}{b_k}\\
	&= \tilO{\frac{SA}{\eta} + \lambda DTK}  +\rbr{\frac{1+C\sqrt{8\ln\rbr{\frac{CSA}{\delta}}}}{\eta'} + 4CH\ln\rbr{\frac{CSA}{\delta}} - \frac{1}{70\eta\ln K}} \inner{\phi^\star}{\rho_K}\\
	&\qquad + (40\eta + 2\gamma - \lambda)\sumk\inner{q_k}{\h{c_k}} + \rbr{\rbr{1+C\sqrt{8\ln\rbr{\frac{CSA}{\delta}}}} \eta'   - \gamma}\sumk\inner{\phi^\star}{b_k}.
\end{align*}
Finally, note that $\eta' \geq 0$ from \pref{lem:DEV_1_2} and \pref{lem:DEV_3_4} can be chosen arbitrarily.
Setting $\eta' = \gamma/\rbr{1+C\sqrt{8\ln\rbr{\frac{CSA}{\delta}}}}$,
and plugging the choice of $\gamma = 100\eta\ln K\rbr{1+C\sqrt{8\ln\frac{CSA}{\delta}}}^2$ and $\lambda = 40\eta+2\gamma$,
one can see that the coefficients multiplying the last three terms $\inner{\phi^\star}{\rho_K}$, $\sumk\inner{q_k}{\h{c_k}}$, and $\sumk\inner{\phi^\star}{b_k}$ are all non-positive.
Therefore, we arrive at
\[
R_K = \tilO{\frac{SA}{\eta} + \eta DTK\ln\rbr{\nicefrac{1}{\delta}}}
= \tilO{\sqrt{DTSAK\ln\rbr{\nicefrac{1}{\delta}}}},
\]
where the last step is by the choice of $\eta$.
\end{proof}

\begin{lemma}\label{lem:OMD-increasing-eta}
\pref{alg:bandit-known-hp} ensures with probability at least $1-\delta$:
\[
\reg_\qfeat \leq \tilO{\frac{SA}{\eta}} - \frac{\inner{\phi^\star}{\rho_K}}{70\eta\ln K} + 40\eta\sumk\inner{q_k}{\h{c_k}} + \tilO{H^2\sqrt{SA}\ln\rbr{\nicefrac{1}{\delta}}}.
\]
\end{lemma}
\begin{proof}
	Denote by $n(s, a)$ the number of times the learning rate for $(s, a)$ increases, such that $\eta_K(s, a)=\eta\beta^{n(s, a)}$,
	and by $k_1, \ldots, k_{n(s, a)}$ the episodes where $\eta_k(s, a)$ is increased, such that $\eta_{k_{t}+1}(s, a)=\beta\cdot\eta_{k_{t}}(s, a)$.
	Since $\rho_1(s,a) = 2T$ and
	\[
	  \rho_1(s, a)2^{n(s, a)-1}\leq \cdots \leq 
	\rho_{k_{n(s,a)}}(s, a) \leq \frac{1}{\qfeat_{k_{n(s,a)}+1}(s,a)} 
	\leq \frac{1}{q_{k_{n(s,a)}+1}(s,a)} \leq TK^4,
	\]
	we have $n(s, a)\leq 1+ \log_2\frac{K^4}{2}\leq 7\log_2K$. Therefore, $\eta_K(s, a)\leq \eta e^{\frac{7\log_2K}{7\ln K}}\leq 5\eta$.
	
	Now, notice that 
	\[
	\gamma\hatb_k(s,a) \leq \frac{\gamma H \sum_h q_k(s,a,h)\hatc_k(s,a)}{q_k(s,a)}
	= \gamma H \hatc_k(s,a) \leq  \hatc_k(s,a).
	\]
	This means that the cost $\hatc_k - \gamma\hatb_k$ we feed to OMD is always non-negative, and thus by the same argument of~\citep[Lemma~12]{agarwal2017corralling}, we have
	\begin{align*}
		&\reg_\qfeat = \sumk\inner{\qfeat_k-\qfeat^\star}{\hatc_k-\gamma\hatb_k}\\
		&\leq \sumk D_{\psi_k}(\qfeat^\star, \qfeat_k) - D_{\psi_k}(\qfeat^\star, \qfeat_{k+1}) + \sumk\sumsa\eta_k(s, a)\qfeat^2_k(s, a)(\hatc_k(s, a) - \gamma\hatb_k(s, a))^2\\
		&\leq D_{\psi_1}(\qfeat^\star, \qfeat_1) + \sum_{k=1}^{K-1}\rbr{D_{\psi_{k+1}}(\qfeat^\star, \qfeat_{k+1}) - D_{\psi_k}(\qfeat^\star, \qfeat_{k+1})} + 5\eta\sumk\sumsa \qfeat_k^2(s, a)\hatc_k^2(s, a)\\
		&\leq D_{\psi_1}(\qfeat^\star, \qfeat_1) + \sum_{k=1}^{K-1}\rbr{D_{\psi_{k+1}}(\qfeat^\star, \qfeat_{k+1}) - D_{\psi_k}(\qfeat^\star, \qfeat_{k+1})} + 20\eta\sumk\sumsa q_k^2(s, a)\hatc_k^2(s, a)\\
		&= D_{\psi_1}(\qfeat^\star, \qfeat_1) + \sum_{k=1}^{K-1}\rbr{D_{\psi_{k+1}}(\qfeat^\star, \qfeat_{k+1}) - D_{\psi_k}(\qfeat^\star, \qfeat_{k+1})} + 20\eta\sumk\sumsa\tilN_k^2(s, a)c_k^2(s, a).
	\end{align*}
	For the first term, since $\qfeat_1$ minimizes $\psi_1$ and thus $\inner{\nabla\regz_1(\qfeat_1)}{\qfeat^\star-\qfeat_1}\geq 0$, we have
	\begin{align*}
		D_{\psi_1}(\qfeat^\star, \qfeat_1) \leq \psi_1(\qfeat^\star) - \psi_1(\qfeat_1) = \frac{1}{\eta}\sumsa\ln\frac{\qfeat_1(s, a)}{\qfeat^\star(s, a)} 
		\leq \frac{1}{\eta}\sumsa\ln\frac{2H}{q^\star(s,a)} = \tilO{\frac{SA}{\eta}}.
	\end{align*}
	For the second term, we define $\aux(y) = y-1-\ln y$ and proceed similarly to~\citep{agarwal2017corralling}:
	\begin{align*}
		&\sum_{k=1}^{K-1} D_{\psi_{k+1}}(\qfeat^\star, \qfeat_{k+1}) - D_{\psi_k}(\qfeat^\star, \qfeat_{k+1})\\
		&= \sum_{k=1}^{K-1}\sumsa\rbr{\frac{1}{\eta_{k+1}(s, a)} - \frac{1}{\eta_k(s, a)}}\aux\rbr{\frac{\qfeat^\star(s, a)}{\qfeat_{k+1}(s, a)}}\\
		&\leq\sumsa\frac{1-\beta}{\eta\beta^{n(s, a)}} \aux\rbr{\frac{\qfeat^\star(s, a)}{\qfeat_{k_{n(s, a)}+1}(s, a)}}\\
		&=\sumsa\frac{1-\beta}{\eta\beta^{n(s, a)}}\rbr{\frac{\qfeat^\star(s, a)}{\qfeat_{k_{n(s, a)}+1}(s, a)} - 1 - \ln\frac{\qfeat^\star(s, a)}{\qfeat_{k_{n(s, a)}+1}(s, a)}}\\
		&\leq -\frac{1}{35\eta\ln K}\sumsa\rbr{ \qfeat^\star(s, a)\frac{\rho_K(s, a)}{2} - 1 - \ln\frac{\qfeat^\star(s, a)}{\qfeat_{k_{n(s, a)}+1}(s, a)} }\\
		&\leq \frac{SA(1 + 6\ln K)}{35\eta\ln K} - \frac{\inner{\phi^\star}{\rho_K}}{70\eta\ln K} = \tilO{\frac{SA}{\eta}}- \frac{\inner{\phi^\star}{\rho_K}}{70\eta\ln K},
	\end{align*}
	where in the last two lines we use the facts $1-\beta \leq -\frac{1}{7\ln K}, \beta^{n(s, a)} \leq 5$, $\rho_K(s, a)=\frac{2}{\phi_{k_{n(s, a)}+1}(s, a)}$, and $\ln\frac{q^\star(s, a)}{q_{k_{n(s, a)}+1}(s, a)} \leq \ln(HTK^4) \leq 6\ln K$.
	
	Finally, for the third term, 
	since $\sumsa \tilN_k^2(s, a)c_k^2(s, a) \leq \left(\sumsa \tilN_k(s, a)\right)^2 \leq H^2$, we apply Azuma's inequality (\pref{lem:azuma}) and obtain, with probability at least $1-\delta$:
	\begin{align*}
		\eta\sumk\sumsa \tilN_k^2(s, a)c_k^2(s, a) &\leq \eta\sumk\E_k\sbr{\sumsa \tilN_k^2(s,a )c_k^2(s, a)} + \tilO{\eta H^2\sqrt{K\ln\rbr{\nicefrac{1}{\delta}}}} \\
		&\leq  \eta\sumk\E_k\sbr{\inner{\tilN_k}{c_k}^2} + \tilO{H^2\sqrt{SA}\ln\rbr{\nicefrac{1}{\delta}}}\\
		&\leq 2\eta\sumk\inner{q_k}{\h{c_k}} + \tilO{H^2\sqrt{SA}\ln\rbr{\nicefrac{1}{\delta}}} \tag{\pref{lem:deviation_loop_free}}.
	\end{align*}
	Combining everything shows 
	\begin{align*}
		\reg_\qfeat \leq \tilO{\frac{SA}{\eta}} - \frac{\inner{\phi^\star}{\rho_K}}{70\eta\ln K} + 40\eta\sumk\inner{q_k}{\h{c_k}} + \tilO{H^2\sqrt{SA}\ln\rbr{\nicefrac{1}{\delta}}}.
	\end{align*}
	finishing the proof.
\end{proof}

\begin{lemma}\label{lem:DEV_1_2}
For any $\eta' > 0$, 
with probability at least $1-\delta$, 
\[
\dev_1 \leq C\sqrt{8\ln\rbr{\frac{CSA}{\delta}}}\rbr{\frac{\inner{\qfeat^\star}{\rho_K}}{\eta'}+\eta'\inner{\phi^\star}{\sumk b_k}} + \rbr{4CH\ln\rbr{\frac{CSA}{\delta}}}\inner{\phi^\star}{\rho_K}.
\]
Also, with probability at least $1-\delta$, 
$
\dev_2 = \tilO{ H^2\sqrt{SA}\ln\rbr{\nicefrac{1}{\delta}} }.
$
\end{lemma}
\begin{proof}
	Define $X_k(s, a) = \hatc_k(s, a) - c_k(s, a)$.
	Note that \[
	X_k(s, a)\leq \frac{H}{q_k(s,a)} \leq \frac{2H}{\qfeat_k(s,a)} \leq 2H\rho_k(s, a) \leq 4HTK^4,
	\] and
	\begin{align*}
		\sumk\E_k \sbr{X_k^2(s, a)} &\leq \sumk \frac{\E_k\sbr{\tilN_k^2(s, a)c^2_k(s, a)}}{q^2_k(s, a)}  \\
		&\leq 2\rho_k(s,a) \sumk \frac{\E_k\sbr{\tilN_k^2(s, a)c^2_k(s, a)}}{q_k(s, a)}  \\
		&= 4\rho_k(s,a) \sumk b_k(s,a). \tag{\pref{lem:simpler_form_N_times_c}}
	\end{align*}
	Therefore, by applying a strengthened Freedman's inequality (\pref{lem:extended-freedman}) with $b = 4HTK^4$, $B_k=2H\rho_k(s, a), \max_kB_k=2H\rho_K(s, a)$,  and $V = 4\rho_k(s,a) \sumk b_k(s,a)$, we have with probability $1-\delta/(SA)$, 
	\begin{align*}
		&\sumk \hatc_k(s, a) - c_k(s, a) \\ 
		&\leq C\rbr{ \sqrt{32\rho_K(s, a)\sumk b_k(s, a)\ln\rbr{\frac{CSA}{\delta}}} + 4H\rho_K(s, a)\ln\rbr{\frac{CSA}{\delta}} } \\
		&\leq C\sqrt{8\ln\rbr{\frac{CSA}{\delta}}}\rbr{\frac{\rho_K(s,a)}{\eta'}+\eta'\sumk b_k(s,a)} + 4CH\rho_K(s, a)\ln\rbr{\frac{CSA}{\delta}}, 
	\end{align*}
	where the last step is by AM-GM inequality.
	Further using a union bound shows that the above holds for all $(s, a)\in\tilSA$ with probability $1-\delta$ and thus
	\begin{align*}
		\dev_1 &= \sumk\inner{\qfeat^\star}{\hatc_k-c_k} \\
		&\leq C\sqrt{8\ln\rbr{\frac{CSA}{\delta}}}\rbr{\frac{\inner{\qfeat^\star}{\rho_K}}{\eta'}+\eta'\inner{\phi^\star}{\sumk b_k}} + \rbr{4CH\ln\rbr{\frac{CSA}{\delta}}}\inner{\phi^\star}{\rho_K}.
	\end{align*}
	To bound $\dev_2$, simply note that $\abs{\inner{\h{q_k}}{c_k-\hatc_k}} \leq 2H^2$ and apply Azuma's inequality (\pref{lem:azuma}): with probability $1-\delta$,
	\begin{align*}
		\dev_2 = \lambda\sumk\inner{\h{q_k}}{c_k - \hatc_k}= \bigO{\lambda H^2\sqrt{K\ln\rbr{\nicefrac{1}{\delta}}}} = \tilO{H^2\sqrt{SA}\ln\rbr{\nicefrac{1}{\delta}}}.
	\end{align*}
	This completes the proof.
\end{proof}

\begin{lemma}\label{lem:DEV_3_4}
With probability at least $1-\delta$, we have $\dev_3 = \tilO{H^2\sqrt{SA}\ln\rbr{\nicefrac{1}{\delta}}}$.
Also, for any $\eta' > 0$, 
with probability at least $1-\delta$, we have
\[
\dev_4 \leq \frac{\inner{\qfeat^\star}{\rho_K}}{\eta'}+\eta'\inner{\phi^\star}{\sumk b_k} + \tilO{1}.
\]
\end{lemma}
\begin{proof}
To bound $\dev_3$, simply note that 
\[
\left|\inner{\phi_k}{\hatb_k-b_k}\right|  
\leq 4H \inner{q_k}{\hatc_k}
\leq 4H\left(\sumsa \tilN_k(s, a)\right)  \leq 4H^2  
\] and apply Azuma's inequality: with probability $1-\delta$,
	\begin{align*}
		\dev_3 = \gamma\sumk\inner{\phi_k}{\hatb_k - b_k} = \tilO{\gamma H^2\sqrt{K\ln\rbr{\nicefrac{1}{\delta}}}} = \tilO{H^2\sqrt{SA}\ln\rbr{\nicefrac{1}{\delta}}}.
	\end{align*}
	To bound $\dev_4=\gamma\sumk\inner{\phi^\star}{b_k - \hatb_k}$, we note that $b_k(s, a) - \hatb_k(s,a) \leq b_k(s,a) \leq H$, and
	\begin{align*}
		\E_k\sbr{b_k(s, a) - \hatb_k(s, a)}^2 &\leq \frac{\E_k\sbr{(\sum_h h\cdot q_k(s, a, h)\hatc_k(s, a))^2}}{q_k^2(s, a)}\\
		&\leq \frac{H^2\E_k\sbr{\tilN_k^2(s, a)c_k^2(s, a)}}{q^2_k(s, a)} \\
		&\leq 2H^2\rho_K(s,a)\frac{\E_k\sbr{\tilN_k^2(s, a)c_k^2(s, a)}}{q_k(s, a)} \\
		& \leq 4H^2\rho_K(s, a)b_k(s, a) \tag{\pref{lem:simpler_form_N_times_c}}.
	\end{align*}
	Hence, applying a strengthened Freedman's inequality (\pref{lem:extended-freedman}) with $b = B_i = H$, $V = 4H^2\rho_K(s, a)\sumk b_k(s, a)$, and $C' = \lceil \log_2 H \rceil \lceil \log_2 (H^2K) \rceil$, we have
 with probability at least $1-\delta/(SA)$,
	\begin{align*}
		&\sumk b_k(s, a) - \hatb_k(s, a)   \\
		&\leq  2C'H\sqrt{\ln\rbr{\frac{C'SA}{\delta}}}\sqrt{8\rho_K(s,a) \sumk b_k(s,a)} + 2C'H\ln\rbr{\frac{C'SA}{\delta}} \\
		&=  2C'H\sqrt{\ln\rbr{\frac{C'SA}{\delta}}}\rbr{\frac{2\rho_K(s,a)}{\eta'} + \eta' \sum_k b_k(s,a)} + 2C'H\ln\rbr{\frac{C'SA}{\delta}},
	\end{align*}
	where the last step is by AM-GM inequality.
	Finally, applying a union bound shows that the above holds for all $(s,a)\in\tilSA$ with probability at least $1-\delta$ and thus
	\begin{align*}
		\dev_4 = \gamma\sumk\inner{\phi^\star}{b_k - \hatb_k} \leq \frac{\inner{\qfeat^\star}{\rho_K}}{\eta'}+\eta'\inner{\phi^\star}{\sumk b_k} + \tilO{1},
	\end{align*}
	where we bound $\gamma C'H\sqrt{\ln\rbr{\frac{C'SA}{\delta}}}$ by a constant since $\gamma$ is of order $1/\sqrt{K}$ and is small enough when $K$ is large.
\end{proof}

\begin{lemma}\label{lem:simpler_form_N_times_c}
	%For any episode $k$ and $(s, a)\in\tilSA$, we have $\E_k \sbr{\tilN_k(s, a)^2c_k(s, a)^2} \leq 2q_k(s, a)b_k(s, a)$.
	For any episode $k$ and $(s, a)\in\tilSA$: $\E_k \sbr{\tilN_k(s, a)^2c_k(s, a)^2} \leq 2q_k(s, a)b_k(s, a)$.
\end{lemma}
\begin{proof}
%The proof is similar to those of \pref{lem:deviation} and uses the inequality
The proof is similar to those of \pref{lem:deviation} and uses
$(\sum_{i=1}^I a_i)^2\leq 2\sum_{i}a_i(\sum_{i'= i}^I a_{i'})$:
	\begin{align*}
		\E_k\sbr{ \tilN_k(s, a)^2c_k(s, a)^2} &\leq \E_k\sbr{\rbr{ \sum_{h=1}^H \tilN_k(s, a, h)}^2c_k(s, a) }\\
		&\leq 2\E_k\sbr{\rbr{\sum_{h=1}^H \tilN_k(s, a, h)}\rbr{\sum_{h'\geq h}^H \tilN_k(s, a, h')c_k(s, a)}}\\
		&\leq 2\E_k\sbr{\sum_{h=1}^H\sum_{h'\geq h}^H \tilN_k(s, a, h')c_k(s, a)} \tag{$\tilN_k(s, a, h)\in \{0,1\}$}\\
		&= 2 \sum_{h=1}^H\sum_{h'\geq h}^H q_k(s, a, h')c_k(s, a) \\
		&= 2\sum_{h=1}^Hh  q_k(s, a, h)c_k(s, a) = 2q_k(s, a)b_k(s, a),
	\end{align*}
	where the last step is by the definition of $b_k$.
\end{proof}

\section{Concentration Inequalities}
% !TEX root = main.tex

\begin{lemma}{(Azuma's inequality)}\label{lem:azuma}
	Let $X_{1:n}$ be a martingale difference sequence and $\abs{X_i} \leq B$ holds for $i=1,\ldots, n$ and some fixed $B > 0$.
	Then, with probability at least $1-\delta$:
	\begin{align*}
		\left|\sum_{i=1}^nX_i\right| \leq B\sqrt{2n\ln\frac{2}{\delta}}.
	\end{align*}
\end{lemma}

\begin{lemma}{(A version of Freedman's inequality from~\citep{beygelzimer2011contextual})}\label{lem:freedman}
	Let $X_{1:n}$ be a martingale difference sequence and $X_i \leq B$ holds for $i=1,\ldots, n$ and some fixed $B > 0$. Denote $V=\sum_{i=1}^n\E_i[X_i^2]$.
	Then, for any $\lambda\in [0, 1/B]$, with probability at least $1-\delta$:
	\begin{align*}
		\sum_{i=1}^nX_i \leq \lambda V + \frac{\ln(1/\delta)}{\lambda}.
	\end{align*}
\end{lemma}

\begin{lemma}{(Strengthened Freedman's inequality from~\citep[Theorem 2.2]{lee2020bias})}\label{lem:extended-freedman}
Let $X_{1:n}$ be a martingale difference sequence with respect to a filtration $\calF_1 \subseteq \cdots \subseteq \calF_n$ such that $\E[X_i |\calF_i] = 0$.
Suppose $B_i \in [1,b]$ for a fixed constant $b$ is $\calF_i$-measurable and such that $X_i \leq B_i$ holds almost surely.
Then with probability at least $1 -\delta$ we have 
\[
    \sum_{i=1}^n  X_i \leq  C\big(\sqrt{8V\ln\left(C/\delta\right)} + 2B^\star  \ln\left(C/\delta\right)\big),
\]
where $V = \max\big\{1, \sum_{i=1}^n \E[X_i^2 | \calF_i]\big\}$, $B^\star =\max_{i\in[n]} B_i$, and 
$C = \ceil{\ln(b)}\ceil{\ln(nb^2)}$.
\end{lemma}

\end{document}